\documentclass[11pt]{article}

\usepackage[preprint]{acl}

\usepackage{times}
\usepackage{latexsym}
\usepackage{adjustbox}

\usepackage[T1]{fontenc}

\usepackage[utf8]{inputenc}

\usepackage{microtype}
\usepackage{amsmath}
\usepackage{amssymb}
\usepackage{amsthm}
\usepackage{booktabs}
\usepackage{multirow}

\usepackage{inconsolata}

\usepackage{graphicx}

\newcommand{\Rmn}[1]{\uppercase\expandafter{\romannumeral #1\relax}}
\newtheorem{definition}{Definition}
\newtheorem{assumption}{Assumption}
\newtheorem{proposition}{Proposition}
\newtheorem{corollary}{Corollary}
\newtheorem{theorem}{Theorem}
\newtheorem{lemma}{Lemma}
\theoremstyle{remark}
\newtheorem{remark}{Remark}

\title{Model-Specific Task Similarity for Vision–Language Model Selection via Layer Conductance}

\author{
  Wei Yang \\
  School of AI and Data Science \\
  USTC \\
  \And
  Hong Xie\thanks{Corresponding author: \texttt{hongx87@ustc.edu.cn}}   
  {\bf Tao Tan} \\  
  {\bf Xin Li}   
  {\bf Defu Lian} 
  {\bf Enhong Chen} \\
  USTC \\
}

\begin{document}
\maketitle
\begin{abstract}
While open-sourced Vision-Language Models (VLMs) have proliferated, selecting the optimal pre-trained model for a specific downstream task remains challenging. Exhaustive evaluation is often infeasible due to computational constraints and data limitations in few-shot scenarios. Existing selection methods fail to fully address this: they either rely on data-intensive proxies or use symmetric textual descriptors that neglect the inherently directional and model-specific nature of transferability. To address this problem, we propose a framework that grounds model selection in the internal functional dynamics of the visual encoder. Our approach represents each task via layer-wise conductance and derives a target-conditioned block-importance distribution through entropy-regularized alignment. Building on this, we introduce Directional Conductance Divergence (DCD), an asymmetric metric that quantifies how effectively a source task covers the target’s salient functional blocks. This allows for predicting target model rankings by aggregating source-task ranks without labeled target-task evaluation. Experimental results on 48 VLMs across 21 datasets demonstrate that our method outperforms state-of-the-art baselines, achieving a 14.7\% improvement in NDCG@5 over SWAB. 
\end{abstract}

\section{Introduction}
VLMs pretrained on large-scale image--text pairs have become core engines for tasks such as visual understanding and embodied intelligence, particularly in zero-shot transfer scenarios~\citep{vlm—no1,ma_vla,vlm_survey}. With the rapid emergence of open-source VLMs spanning diverse architectures and pretraining strategies, selecting the most suitable pretrained model for a target task has become increasingly critical. Exhaustive evaluation is often infeasible, since target domains usually provide only a few or even no test data samples at all. Therefore, 
transferability estimation becomes the mainstream approach to crack this problem~\cite{Ding2024,Xue2024}.  

Transferability estimation is a central and long-standing 
problem in the transfer learning area~\cite{Ding2024,Xue2024}. 
Though many methods were proposed, 
they fall short in the VLMs selection problem.  
More specifically, data-driven proxy metrics are unstable in few-shot regimes~\cite{logme,leep}.  
Alternatively, prior task-centric formulations 
like LOVM~\cite{lovm_modelgpt} and SWAB~\cite{swab} 
estimate transferability by the embeddings of textual task descriptions. 
Being target domain data free is one major merit of these methods, 
but it also leads to the performance bottleneck.  
This is because task similarity is inherently model-specific~\cite{task2vec,model-specifc-sim}.  
This results in the \textit{data} vs. \textit{performance} dilemma.  



This paper develops an alternative approach to address the aforementioned dilemma.
Our idea is inspired by practical scenarios that 
a few unlabeled target images are usually available 
in the target task. 
Though these unlabeled target images offer relatively weak visual signals, 
it turns out that they can elicit strong signals for estimating model-specific task similarity if done properly.  
To leverage them, we ground task relatedness in model-specific internal dynamics. Concretely, we represent each task by the layer-wise contribution profile of the model's visual encoder, computed via layer conductance over coarse-grained blocks~\citep{layerconductance1}. This yields a lightweight, label-free, and inherently model-specific task representation, directly reflecting which functional blocks the target domain relies on. 


Building on the above view, we derive a target-conditioned block-importance distribution via entropy-regularized alignment~\cite{entro_ot} and propose Directional Conductance Divergence (DCD), an asymmetric metric that captures whether source-task mechanisms cover target-critical blocks. Note the asymmetric property is desirable, 
since transfer is inherently directional and model-specific: tasks that appear similar under one model's representations may diverge under another~\cite{task2vec}. We then predict target-task rankings by aggregating source-task model rankings with DCD-induced similarity weights, avoiding labeled target-task evaluation while relying on amortizable source-side computation. Under limited sampling, our approach achieves roughly a 15\% improvement in NDCG@5 and $\tau$@5 over SWAB.

The main contributions of this work are threefold. 
(\Rmn{1}) We introduce a model-specific task representation based on layer conductance of the visual encoder, which requires only a few unlabeled images from the target task. 
(\Rmn{2}) We propose Directional Conductance Divergence (DCD), a target-conditioned asymmetric transferability metric that measures whether source tasks match the target's critical functional blocks. 
(\Rmn{3}) We develop a similarity-weighted rank aggregation scheme to predict target-task model rankings without direct target-task evaluation, and validate its effectiveness across diverse VLMs and datasets.

\section{Related Work}

{\bf Vision-Language Models.} VLMs, trained via contrastive learning on large-scale image--text pairs, have become foundational to cross-modal representation learning~\cite{vlm—no1}. 
Prominent examples include CLIP~\cite{vlm—no1}, ALIGN~\cite{align}, FLAVA~\cite{flava}, Florence~\cite{florence}, and CoCa~\cite{coca}. Despite these advances, the performance of VLMs varies considerably across tasks and domains and often degrades under distribution shift~\cite{lovm_modelgpt}. Meanwhile, the rapid proliferation of open-source models, where each trained on distinct architectures and corpora, has given rise to a heterogeneous ``VLMs Zoo''~\cite{vlmzoo,open_clip,hf}. In this landscape, strong benchmark results do not guarantee downstream effectiveness, shifting the central challenge from training a single universal model to selecting the most suitable VLM for a given target task. This work investigates label-free and reliable VLM selection from the zoo without exhaustive labeled per-task evaluation.

{\bf Model Selection.} 
A key challenge in deploying pre-trained models lies in estimating their transferability. Existing methods address it via proxy metrics such as H-Score~\cite{h-score}, LogME~\cite{logme} and LEEP~\cite{leep}, 
or through representation similarity such as 
Task2Vec~\cite{task2vec} and Model Spider~\cite{model_spider}. Unsupervised approaches further attempt to predict performance using only unlabeled data signals~\cite{unsuper1}. However, these methods uniformly demand sufficient target-task data.  This demand ill-suits VLMs, which are often deployed in zero-shot or few-shot scenarios~\cite{lovm_modelgpt,swab}. Moreover, most prior work is developed for unimodal architectures and rarely accounts for VLMs' multimodal nature.

Recent paradigms decouple VLM evaluation from labeled image. LOVM~\cite{lovm_modelgpt} ranks VLMs using only textual task descriptions, while SWAB~\cite{swab} advances this via an optimal transport framework that bridges modality and capability gaps across tasks. Parallel efforts like Model Label Learning (MLL)~\cite{tan2025vision} further propose pre-assigning capability labels to VLMs to enable efficient selection and ensemble-based reuse for downstream adaptation. These methods pushed the frontier of state of the art, 
but they also have limitations worth further investigation: they ignore the practical prevalence of few-shot settings where a small set of target images is available. In this regime, data-heavy methods yield unstable estimates, whereas text-only approaches waste informative visual cues already present in the available samples.

VLMs performance hinges on both visual representation quality and vision--language alignment~\cite{vlm—no1,lawofvi}. When limited target-domain images are available, leveraging these samples to reveal performance differences among candidate models becomes essential for reliable selection—a setting that remains largely unexplored.  
Furthermore, prior task-centric formulations (LOVM~\cite{lovm_modelgpt} and SWAB~\cite{swab}) embed tasks into text space to estimate transferability. This formulation overlooks the fact that task similarity is inherently model-specific~\cite{task2vec,model-specifc-sim}: tasks that appear similar under one model’s representations may diverge substantially under another, due to differences in architectural biases, pre-training objectives, and learned features. Recent empirical analysis~\cite{s2025understanding} reinforces this by mapping the complex, non-uniform patterns of positive and negative transfer across different VLM architectures. Ignoring this model-specific nature of task relatedness therefore leads to inaccurate transferability estimation, particularly across heterogeneous VLM collections. Motivated by these observations, we propose a VLMs selection framework tailored to few-shot and weakly supervised settings, which leverages limited visual evidence while explicitly modeling task relatedness in a model-specific manner.

{\bf Layer Conductance.} 
Layer conductance is an attribution method that quantifies the contribution of intermediate-layer activations to model outputs~\cite{layerconductance1,layerconductance2}. It extends Integrated Gradients from input features to hidden representations, enabling principled attribution while satisfying axiomatic properties such as completeness.
We instantiate our framework with layer conductance because it provides block-level attribution signals for internal functional dynamics. Our framework is not inherently restricted to layer conductance, and alternative attribution or representation-analysis methods may provide complementary task representations.

\section{Unlabeled Images Elicit Similarity Signals}
\label{sec:qualitative}
This section motivates and validates the core premise of our approach that a small set of unlabeled target images can reveal reliable, model-specific signals of task relatedness, which align more closely with transferability than text-only semantic proxies.

\paragraph{Task Representations via Layer Conductance.}
We consider the setting where a few unlabeled images are available.  
A classical view in computer vision treats deep neural networks as cascaded filter banks that learn hierarchical representations through successive transformations~\cite{filter1, filter2}. Under this perspective, different layers capture distinct aspects of task-relevant information. If a task consistently relies on specific blocks (e.g., early texture-oriented stages versus later semantic stages), the distribution of layer contributions provides a compact summary of the task as perceived by the model. This intuition resonates with the task embedding paradigm of Task2Vec~\cite{task2vec}, which represents tasks through statistics of internal learning signals. We extend this to a model-specific formulation: rather than constructing a universal task graph, we characterize each task by its layer-wise dependency structure within a given VLM. Specifically, we use layer conductance~\cite{layerconductance1, layerconductance2} to extract a block-wise contribution vector from the visual encoder. We use this distribution as a model-specific summary of how a task engages internal visual blocks, providing a basis for estimating task relatedness from the model's perspective.  

\begin{figure}[t]
    \centering
    \includegraphics[width=\linewidth]{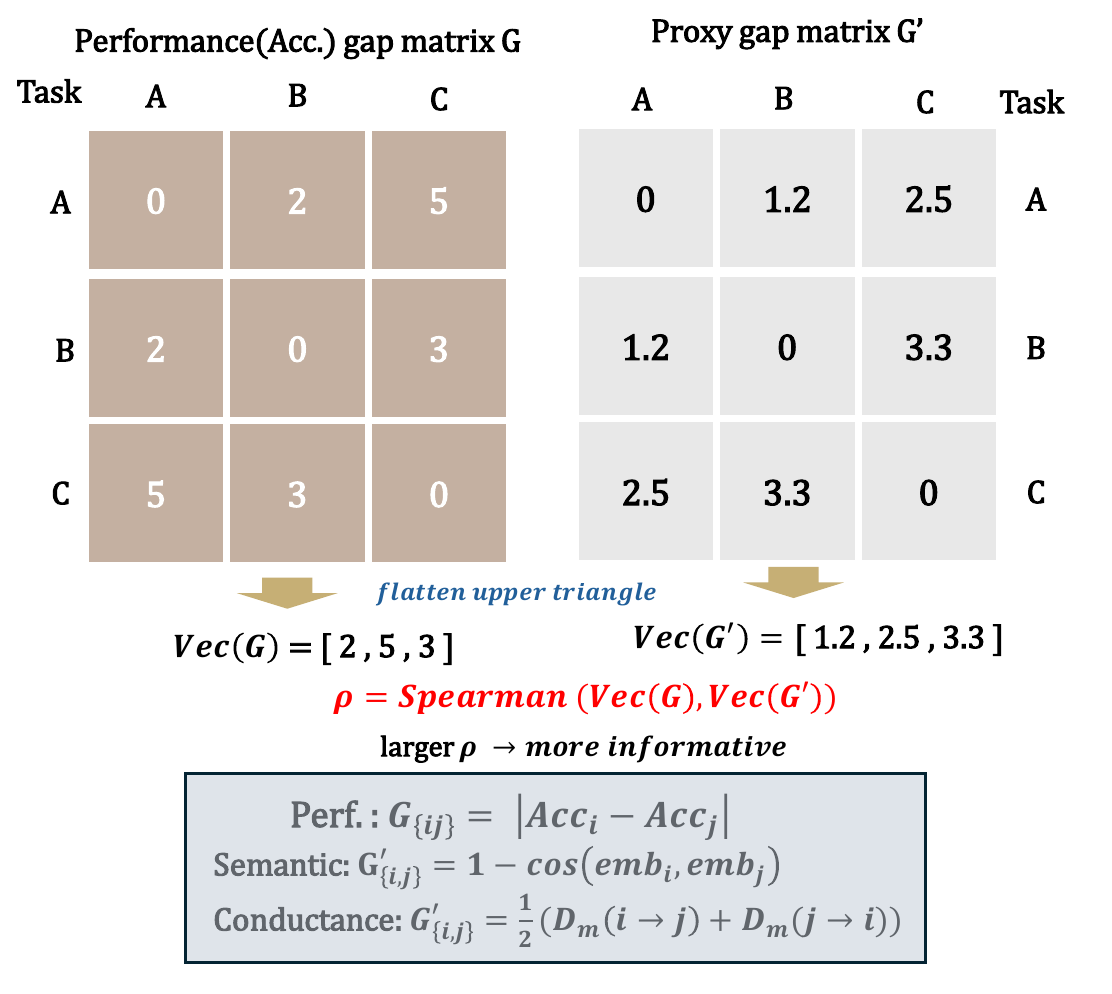}
    \caption{\textbf{How we measure proxy reliability.} 
    We construct a performance gap matrix $G$ with entries $G_{ij}=|\mathrm{Acc}_i-\mathrm{Acc}_j|$, and a proxy gap matrix $G'$ from either semantic distances or conductance-induced distances. }
    \label{fig:easy_example}
\end{figure}

\paragraph{Alignment with Performance.}
We first ask: does a similarity proxy preserve the true structure of transferability? If two tasks are highly related from a model's perspective, the performance gap between them should be small~\cite{perf_gap}. A useful proxy should therefore assign smaller distances to task pairs with smaller performance gaps. To probe this, we select five representative VLMs spanning ResNet and ViT backbones with diverse pretraining sources and compute a task--task proxy matrix using: (i) our conductance representation derived from a 25-image unlabeled set, and (ii) a semantic baseline derived from task text embeddings. Figure~\ref{fig:easy_example} illustrates these matrices and the corresponding performance gaps, with further details deferred to Appendix~\ref{ap:qi_details}.   While absolute correlations are modest, as expected under extreme supervision constraints, the conductance proxy consistently aligns more closely with ground truth than the semantic baseline.

\begin{table}[t]
    \centering
    \small
    \begin{tabular*}{\columnwidth}{@{\extracolsep{\fill}} lccc}
        \toprule
        \textbf{Model} & $\rho_{\text{cond}}$ (Ours) & $\rho_{\text{sem}}$ & $\Delta$ \\
        \midrule
        RN101-OpenAI    & \textbf{0.128} & -0.127 & \textbf{+0.255} \\
        RN50x4-OpenAI   & \textbf{0.131} & -0.095 & \textbf{+0.227} \\
        ViT-B/32-OpenAI & \textbf{0.011} & -0.051 & +0.062 \\
        RN50-CC12M      & \textbf{0.023} & 0.002  & +0.020 \\
        ViT-B/16-DFN2B  & \textbf{0.031} & 0.026  & +0.006 \\
        \bottomrule
    \end{tabular*}
    \caption{\textbf{Alignment with Performance Gap.} Our method ($\rho_{\text{cond}}$) consistently improves over semantic proxies ($\rho_{\text{sem}}$), especially for ResNets where semantic correlation can be negative.}
    \label{tab:correlation_comparison}
\end{table}

Table~\ref{tab:correlation_comparison} quantifies this effect via Spearman correlation between the proxy-derived rankings and the performance gap ordering. Notably, for ResNet-based models (RN101-OpenAI and RN50x4-OpenAI), the semantic proxy can even become negatively correlated with ground truth, suggesting that text-only descriptors may invert the transferability signal under certain architectures. In contrast, the conductance proxy remains positively correlated across all five models, confirming its reliability as a selection signal.

\begin{figure}[t]
    \centering
    \includegraphics[width=0.95\columnwidth]{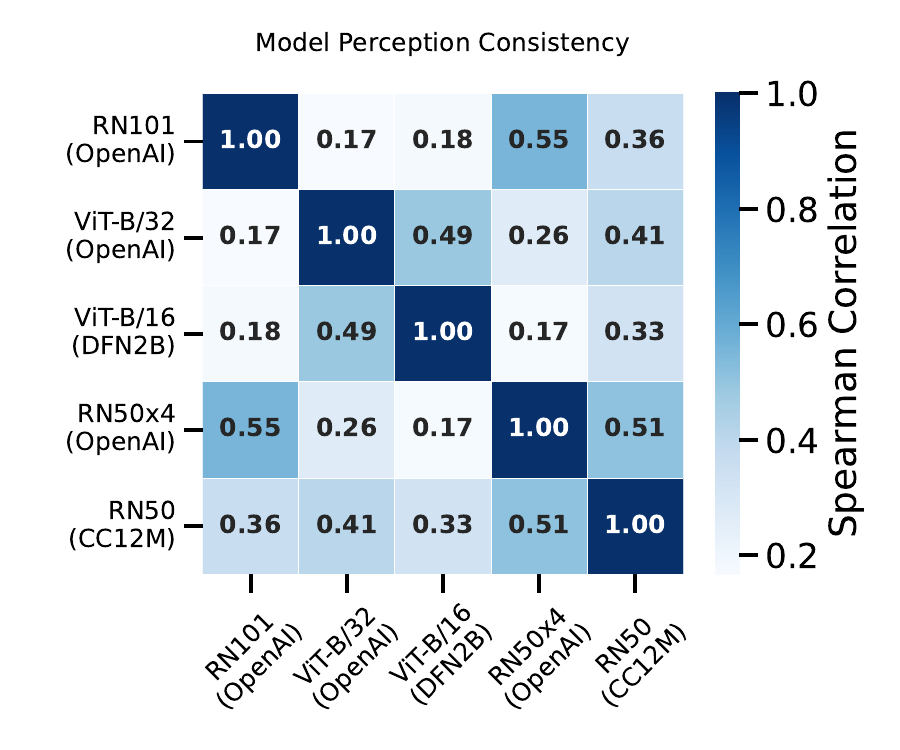}
    \caption{\textbf{Model-to-Model Correlation Matrix in Task Perception.}  High correlations within architectural families and lower correlations across them suggest that conductance-induced task perception is model-dependent, while also indicating that architecture and pretraining can influence the induced task geometry.}
    \label{fig:conductance_heatmap}
\end{figure}

\paragraph{The Model-Specific Nature of Task Perception.}
A second hypothesis central to our method is that task relatedness is model-specific: two tasks perceived as similar by one model may be distinct to another due to differences in architecture, pretraining objectives, and learned features~\cite{task2vec, model-specifc-sim}. If true, the task perception induced by conductance should vary across models. To test this, we compare conductance-induced task distance vectors. Figure~\ref{fig:conductance_heatmap} reports the resulting model-to-model correlation matrix. We observe clear architectural clustering: models within the same backbone family exhibit stronger agreement in relative task distances, while cross-architecture agreement drops substantially. This pattern suggests that conductance-induced task perception is model-dependent, while also indicating that architecture and pretraining can influence the induced task geometry. Rather than treating this clustering as sufficient evidence of task relatedness, we use it as motivation for a model-specific formulation in Section~\ref{sec:method}, where target-conditioned importance and directional divergence are computed individually for each candidate model.

\section{Model Selection Via Unlabeled Image}
\label{sec:method}

Inspired by the insights of Section \ref{sec:qualitative}, 
this section designs a method to select appropriate 
VLMs for downstream tasks utilizing 
a few unlabeled images from the target task. 
\subsection{Preliminary}
\label{sec:preliminary}

{\bf The Model.}
Let $\mathcal{M}$ be a finite set of candidate VLMs and $\mathcal{T}$ a set of downstream tasks. For a model $m \in \mathcal{M}$, we decompose its visual encoder into $d_m$ coarse-grained functional blocks (e.g., stages in ResNet or transformer blocks in ViT), indexed by $i \in \{1, \dots, d_m\}$. The number of blocks $d_m$ may vary across models.
Each task $\tau \in \mathcal{T}$, is associated with a small unlabeled image set $\mathcal{X}_{\tau}=\{x_j\}_{j=1}^{N_\tau}$ sampled i.i.d.\ from its data distribution $P_\tau$. In the leave-one-out protocol, the held-out target task $\tau$ provides $N_{\mathrm{tgt}}$ unlabeled images, while each source task $\sigma \in \mathcal{T}\setminus{\tau}$ provides $N_{\mathrm{src}}$ images for constructing task representations (Eq.~\ref{eq:task_representation}).

\paragraph{Layer Conductance Vector.}
To quantify block-level contributions to the scalar probe, we instantiate our framework with \emph{layer conductance}~\cite{layerconductance1}.
Given an input image $x \sim P_\tau$ for task $\tau \in \mathcal{T}$, we define the layer conductance vector $g_m(x) \in \mathbb{R}^{d_m}$ as the concatenation of block-wise attribution scores:
\begin{equation}
    g_m(x) = \left[ g_m^{(1)}(x), \dots, g_m^{(d_m)}(x) \right]^\top .
\end{equation}
Layer conductance requires a scalar objective. Let $f_m(x)\in\mathbb{R}^{p}$ denote the visual embedding produced by model $m$. In the unlabeled setting, we use the $\ell_2$ norm of the image embedding as a simple label-free scalar probe for inducing conductance profiles:
\begin{equation}
    F_m(x) \triangleq \|f_m(x)\|_2 .
\end{equation}
This choice is architecture-agnostic and aggregates information across embedding dimensions without requiring labels or class semantics.
In practice, Layer conductance produces an attribution tensor for each block output; we summarize it by the mean absolute attribution over activation dimensions to obtain a scalar score $g_m^{(i)}(x)\in\mathbb{R}_{\ge 0}$.
The resulting $g_m(x)$ maps $x$ into a model-specific layer contribution space. More details are in Appendix~\ref{layerconduct}.
Note that $F_m$ is used only as an unsupervised scalar probe to extract relative block-importance patterns, rather than as a direct proxy for downstream accuracy or transferability.

\paragraph{Task Representation.}
Since a task is characterized by its underlying distribution rather than individual samples, we define the task representation $v_{m,\tau}$ as the expectation of the conductance vector over $P_\tau$:
\begin{equation}
    v_{m,\tau} \triangleq \mathbb{E}_{x \sim P_\tau}[g_m(x)] \in \mathbb{R}^{d_m}.
\end{equation}
Since $P_\tau$ is inaccessible in practice, we approximate $v_{m,\tau}$ by an empirical estimate $\hat{v}_{m,\tau}$ computed from a finite unlabeled sample set $\mathcal{X}_\tau=\{x_j\}_{j=1}^{N}$ drawn i.i.d.\ from $P_\tau$:
\begin{equation}
\label{eq:task_representation}
    \hat{v}_{m,\tau} \triangleq \frac{1}{N} \sum_{j=1}^{N} g_m(x_j).
\end{equation}
$N$ denotes the number of available unlabeled images for task $\tau$, which equals $N_{\mathrm{src}}$ for source tasks and $N_{\mathrm{tgt}}$ for the held-out target task.

This vector captures the layer-wise dependency structure of task $\tau$ on the internal blocks of model $m$, and serves as the foundation for the subsequent similarity analysis. For notational convenience, we use $v_{m,\tau}^{(i)}$ to denote the $i$-th component of $v_{m,\tau}$.

\subsection{Target-Conditioned Importance Distribution}
\label{sec:importance_distribution}

Given the task representation $v_{m,\tau}$, we aim to identify which blocks are most critical for the target task $\tau$. Treating all blocks uniformly neglects the hierarchical structure of deep networks, where different blocks capture features of varying semantic complexity. We therefore seek an importance distribution $\alpha_{m,\tau} \in \Delta^{d_m-1}$ over the $d_m$ blocks that emphasizes the most relevant blocks while preserving robustness.

\paragraph{Normalization.}
To eliminate scale variations across models and ensure numerical stability, we define the normalized representation $u_{m,\tau} \in \mathbb{R}^{d_m}$ as:
\begin{equation}
    u_{m,\tau} \triangleq \frac{v_{m,\tau}}{\max(\|v_{m,\tau}\|_2, \epsilon)},
\end{equation}
where $\epsilon > 0$ is a small constant for numerical stability. Each component $u_{m,\tau}^{(i)}$ reflects the relative contribution of block $i$ to task $\tau$.

\paragraph{Derivation via Entropy-Regularized Alignment.}
We require the importance distribution $\alpha_{m,\tau}$ to satisfy two desiderata: (\Rmn{1}) higher weights for blocks with larger normalized conductance values (alignment), and (\Rmn{2}) sufficient dispersion to avoid over-reliance on any single block (robustness).

Following these principles, we define $\alpha_{m,\tau}$ as the solution to an entropy-regularized alignment objective. Specifically, we seek $\alpha \in \Delta^{d_m-1}$ that maximizes a trade-off between alignment with $u_{m,\tau}$ and Shannon entropy $H(\alpha) = -\sum_i \alpha_i \log \alpha_i$:
\begin{equation}
    \alpha_{m,\tau}
    = \mathop{\arg\max}_{\alpha \in \Delta^{d_m-1}}
    \left(
        \langle \alpha, u_{m,\tau} \rangle
        + \frac{1}{\eta} H(\alpha)
    \right),
\end{equation}
where $\eta > 0$ is a temperature parameter controlling distribution sharpness.

This objective admits a closed-form solution given by the softmax function, which can be easily verified (see Appendix~\ref{ap:importance_weight} for a detailed derivation):
\begin{equation}
    \alpha_{m,\tau}^{(i)}
    = \frac{\exp\!\left(\eta \, u_{m,\tau}^{(i)}\right)}
    {\sum_{j=1}^{d_m} \exp\!\left(\eta \, u_{m,\tau}^{(j)}\right)}.
\end{equation}

The resulting distribution depends solely on the target task representation $v_{m,\tau}$ and re-weights blocks according to task demands. The parameter $\eta$ modulates the degree of concentration: large $\eta$ focuses mass on dominant blocks, while small $\eta$ yields a smoother weighting.

\subsection{Directional Conductance Divergence}
\label{sec:dcd}

We define a metric to quantify the transferability from a source task $\sigma$ to the target task $\tau$. Unlike conventional symmetric metrics (e.g., Euclidean or Cosine distance), effective transfer learning hinges on directional coverage~\cite{assym}: a source task is valuable primarily if it encompasses the internal mechanisms critical to the target.

To capture this asymmetry, we propose Directional Conductance Divergence (DCD). We first define the element-wise relative deviation of the source task $\sigma$ with respect to the target $\tau$ at block $i$:
\begin{equation}
    \delta_m^{(i)}(\tau, \sigma) \triangleq 
    \frac{\left| v_{m,\tau}^{(i)} - v_{m,\sigma}^{(i)} \right|}{\left|v_{m,\tau}^{(i)}\right| + \epsilon},
\end{equation}
where $\epsilon > 0$ is a small constant for numerical stability. The denominator normalizes by the target activation magnitude, so the divergence is interpreted as a relative error measured in the target's context.

We then define DCD as the expected relative deviation under the target's importance distribution:
\begin{equation*}
\begin{aligned}
    D_m(\tau \to \sigma)
    &\triangleq \sum_{i=1}^{d_m} \alpha_{m,\tau}^{(i)} \cdot \delta_m^{(i)}(\tau, \sigma) \\
    &= \mathbb{E}_{i \sim \alpha_{m,\tau}}\!\left[ \delta_m^{(i)}(\tau, \sigma) \right]
\end{aligned}
\end{equation*}
Intuitively, $D_m(\tau \to \sigma)$ measures how well the source task $\sigma$ covers the target-critical mechanisms emphasized by $\alpha_{m,\tau}$: smaller values indicate better coverage.

\paragraph{From Divergence to Similarity}
The DCD $D_m(\tau \to \sigma)$ is a nonnegative, target-conditioned divergence. To aggregate evidence across multiple candidate sources, we convert DCD values into a normalized similarity distribution over the source-task set.

Let $\mathcal{S}_\tau \subseteq \mathcal{T}\setminus\{\tau\}$ denote the set of available source tasks for target $\tau$.
For a fixed model $m$, we define a softmin-based similarity score:
\begin{equation}
    s_m(\sigma \mid \tau) \triangleq e^{-\gamma ( D_m(\tau \to \sigma) - \min_{\sigma' \in \mathcal{S}_\tau} D_m(\tau \to \sigma') )}
\end{equation}
where $\gamma > 0$ is a temperature parameter modulating the sharpness of the distribution. Subtracting the minimum DCD provides a numerically stable baseline shift without changing the relative ordering across $\sigma$.

We normalize the similarity scores to obtain a probability distribution over source tasks:
\begin{equation}
    p_m(\sigma \mid \tau)
    \triangleq
    \frac{s_m(\sigma \mid \tau)}{\sum_{\sigma' \in \mathcal{S}_\tau} s_m(\sigma' \mid \tau)}
\end{equation}
Thus, $p_m(\sigma \mid \tau)$ assigns larger mass to source tasks with smaller DCD, and the distribution becomes more concentrated as $\gamma$ increases.

\paragraph{Ranking via Similarity Aggregation}
Given the DCD-induced weights $\{p_m(\sigma \mid \tau)\}_{\sigma \in \mathcal{S}_\tau}$, we estimate how model $m$ would rank on the target task $\tau$.

Assume that for each model $m$ and source task $\sigma \in \mathcal{S}_\tau$, we have access to its performance-based rank $R_m(\sigma)$ among all candidate models (smaller is better). That is, $R_m(\sigma)$ is obtained by sorting models by their ground-truth performance on $\sigma$ and taking the resulting rank value.
We then predict the target-task rank of model $m$ via a similarity-weighted average of source-task ranks:
\begin{equation}
    \widehat{R}_m(\tau)
    \triangleq
    \sum_{\sigma \in \mathcal{S}_\tau} p_m(\sigma \mid \tau)\cdot R_m(\sigma)
\end{equation}
Finally, given a candidate model set $\mathcal{M}$, we obtain the estimated ranking for target task $\tau$ by sorting models in ascending order of $\widehat{R}_m(\tau)$ (smaller is better):
\begin{equation}
    \widehat{\pi}_\tau
    \triangleq
    \mathrm{argsort}_{m \in \mathcal{M}} \, \widehat{R}_m(\tau)
\end{equation}
This ranking serves as the basis for model selection and is evaluated against ground-truth rankings in our experiments.
The theoretical motivation for asymmetry and the full derivations are provided in Appendix~\ref{sec:theory}.

\section{Experiments}
\subsection{Experimental Setup}
\paragraph{Datasets and Model Zoo.}
Building upon the protocols established in the LOVM benchmark~\cite{lovm_modelgpt} and the expansion strategies adopted by SWAB~\cite{swab}, we establish our experimental setup using their publicly available datasets and model repositories. To facilitate a more comprehensive evaluation, we augment the original model zoo with additional representative architectures, including EVA01-g-14~\cite{eva} and nllb-clip-base~\cite{nlibclip}, to expand the collection to a total of 48 models. This extensive model zoo encapsulates a broad spectrum of architectural designs, pre-training corpora, and training paradigms. Detailed information on the datasets and models is provided in Appendix~\ref{ap:modelanddataset}.

\paragraph{Evaluation Protocol.}
To ensure a rigorous and equitable comparison, we adhere to the leave-one-out evaluation protocol employed by ModelGPT and SWAB~\cite{swab,lovm_modelgpt}. Specifically, we iteratively designate one of the 21 datasets as the \textit{held-out} target domain, while the remaining datasets constitute the source pool. All datasets are obtained from official repositories and evaluated using standard test splits.

To establish ground-truth rankings, we leverage the text templates provided by LOVM~\cite{lovm_modelgpt} to construct zero-shot classifiers and measure the Top-1 accuracy of each VLM on the respective test splits. To evaluate robustness under low-shot sampling, we repeat each experiment over 10 independent random draws of unlabeled images. In each run, we resample the target and source images according to the specified $N_{\text{tgt}}$ and $N_{\text{src}}$, recompute conductance-based task representations, and obtain the final ranking. We report the mean and standard deviation across runs.
\paragraph{Metrics.}
Following the baseline evaluation protocol~\cite{lovm_modelgpt,swab}, we report top-$k$ diagnostics: NDCG@5, Kendall's $\tau$@5, NDCG@3, and NDCG@7. 
NDCG@5 is more discriminative than Top-5 Recall used in prior studies, as it captures graded relevance and relative ordering rather than set overlap only. 
We report metric definitions in Appendix~\ref{ap:metrics} and standard NDCG@5/Recall@1 in Appendix~\ref{ap:standard_metrics}.
Additional evaluation results, including hyperparameter sensitivity, sampling efficiency, model clustering, and probe variation, are reported in Appendix~\ref{ap:additional_results}.

\subsection{Main Performance Comparison}
\paragraph{Baselines.} We benchmark our method against several representative baselines from the model selection literature. Specifically, we consider two ranking-based approaches: INB, which ranks VLMs based on their ImageNet performance, and AvgRank, which calculates the mean rank across the remaining 20 source datasets under the leave-one-out protocol. We also include GroupRank, a coarse domain/task grouping baseline. For each target dataset, GroupRank selects source datasets that share the same domain/task group as annotated in Table~\ref{tab:datasets}, and predicts each model's target rank by averaging its ground-truth ranks over these grouped source datasets. Furthermore, we compare our approach with two learning-based baselines: ModelGPT~\cite{lovm_modelgpt} and SWAB~\cite{swab}, the current state-of-the-art method.

\begin{table*}[t]
  \centering
  \small
  \begin{tabular*}{\textwidth}{@{\extracolsep{\fill}} lcccccc}
    \toprule
    \textbf{Metric} & \textbf{INB} & \textbf{AvgRank} & \textbf{ModelGPT} & \textbf{SWAB} & \textbf{GroupRank} & \textbf{Ours} \\
    \midrule
    NDCG@5 $\uparrow$       & 0.539 & 0.524 & 0.544 & \underline{0.616} & 0.609 & \textbf{0.707} \\
    $\tau$@5 $\uparrow$     & 0.268 & 0.254 & 0.267 & \underline{0.318} & 0.270 & \textbf{0.365} \\
    \midrule
    Sum $\uparrow$          & 0.807 & 0.778 & 0.811 & \underline{0.934} & 0.879 & \textbf{1.072} \\
    \midrule
    \multicolumn{7}{l}{\emph{Other metrics:}} \\
    \addlinespace[1pt]
    \hspace{1em}NDCG@3 $\uparrow$ & 0.384 & 0.386 & 0.386 & \underline{0.393} & 0.366 & \textbf{0.402} \\
    \hspace{1em}NDCG@7 $\uparrow$ & 0.651 & 0.699 & 0.702 & 0.691 & \underline{0.741} & \textbf{0.752} \\
    \bottomrule
  \end{tabular*}
  \caption{\textbf{Main results on VLM selection.} 
NDCG@5 and $\tau$@5 are averaged over 10 independent random low-shot sampling runs for stochastic methods.
Our method is evaluated in the lightweight regime with $N_{\text{tgt}}{=}1$ target image and $N_{\text{src}}{=}25$ source images per task. 
The corresponding distribution over 10 runs is provided in Appendix~\ref{ap:complete_results}. 
Best in \textbf{bold}, second-best \underline{underlined}.}
  \label{tab:main_results}
\end{table*}

\paragraph{Results.} As shown in Table~\ref{tab:main_results}, we make the following observations.
(\Rmn{1}) Under the baseline evaluation protocol, our method achieves the strongest overall performance among the compared methods, with the NDCG@5 of 0.707 and the highest $\tau$@5 of 0.365. (\Rmn{2}) The performance gap between our method and SWAB reveals the limitation of relying solely on semantic similarity for model selection. Meanwhile, the gap between GroupRank and DCD suggests that manually defined domain/task groups are informative but insufficient to explain the performance gains. By explicitly accounting for model-specific conductance patterns, our attribution-based approach more effectively captures target-conditioned model--task alignment. (\Rmn{3}). The larger gain in NDCG@5 relative to $\tau$@5 indicates that our method is particularly effective at identifying and ranking the best-matched model at the top position, which is crucial for practical deployment.Taken together, these results suggest that effective VLM selection requires a model-specific characterization of task relevance. By grounding similarity in model-aware attribution signals, our approach achieves more accurate model-task alignment with only a few unlabeled target samples and reusable source-task representations.


\subsection{Ablation Study}
\paragraph{Baselines.}
We ablate two key components of DCD: the similarity formulation and the block-weighting strategy. 
For similarity formulation, we replace DCD with two symmetric alternatives, Cosine similarity and Soft-KL/JSD, while keeping the conductance representation and rank aggregation pipeline unchanged. 
For block weighting, we compare our target-conditioned weighting against Uniform weighting and Target-norm weighting under the same directional relative-deviation distance. 
These variants isolate whether the gains come from directional similarity modeling, target-conditioned block importance, or their combination. 
Detailed definitions of all ablation variants are provided in Appendix~\ref{ap:ablation_details}.

\paragraph{Results.}
Table~\ref{tab:ablation} reports the results of both ablation groups.
(\Rmn{1}) In the similarity formulation group, DCD consistently outperforms both symmetric alternatives. Compared with the stronger symmetric baseline, DCD improves NDCG@5 from 0.640 to 0.707 and $\tau$@5 from 0.302 to 0.365, confirming the importance of directional source-to-target coverage.
(\Rmn{2}) Cosine similarity and Soft-KL/JSD perform nearly identically, suggesting that once the similarity is forced to be symmetric, the choice between geometric and distributional comparison has limited impact.
(\Rmn{3}) In the block-weighting group, Uniform weighting already provides a non-trivial signal, while Target-norm weighting improves NDCG@5 from 0.625 to 0.670. 
The full DCD formulation achieves the best overall performance, especially on $\tau$@5, indicating that target-conditioned block importance and directional relative deviation are both important for robust model ranking.

\begin{table}[t]
  \centering
  \small
  \begin{tabular*}{\columnwidth}{@{\extracolsep{\fill}} l ccc}
    \toprule
    \textbf{Variant} & NDCG@5 $\uparrow$ & $\tau$@5 $\uparrow$ & Sum $\uparrow$ \\
    \midrule
    \multicolumn{4}{l}{\emph{Similarity formulation}} \\
    \addlinespace[1pt]
    \hspace{1em}Cosine (Sym.)        & 0.639 & 0.301 & 0.940 \\
    \hspace{1em}Soft-KL/JSD (Sym.)   & 0.640 & 0.302 & 0.942 \\
    \midrule
    \multicolumn{4}{l}{\emph{Weighting strategy}} \\
    \addlinespace[1pt]
    \hspace{1em}Uniform weighting        & 0.625 & 0.254 & 0.879 \\
    \hspace{1em}Target-norm weighting   & 0.670 & 0.243 & 0.913 \\
    \midrule
   \textbf{DCD (Ours)} & \textbf{0.707} & \textbf{0.365} & \textbf{1.072} \\
    \bottomrule
  \end{tabular*}
  \caption{\textbf{Ablation study.}
We evaluate two groups of variants: similarity formulation and block-weighting strategy. 
Cosine and Soft-KL/JSD replace DCD with symmetric similarity measures, while Uniform weighting and Target-norm weighting replace the target-conditioned block weighting used in DCD. 
DCD combines directional relative deviation with target-conditioned block weighting. 
Detailed variant definitions are provided in Appendix~\ref{ap:ablation_details}.}
  \label{tab:ablation}
\end{table}

\section{Conclusion}

We introduced a model-specific framework for few-shot vision-language model selection that characterizes tasks through layer-wise contributions in the visual encoder and quantifies transferability using Directional Conductance Divergence (DCD), an asymmetric metric measuring target-critical feature coverage. Empirical evaluation across 48 VLMs and 21 datasets shows significant improvements over existing baselines, achieving NDCG@5 of 0.707 and $\tau$@5 of 0.365, representing gains of approximately 15\% over SWAB, while theoretical analysis supports the necessity of asymmetry under coverage semantics.

Future work can improve both the generality and practicality of our framework. We will investigate architecture-aware strategies to align functional blocks across heterogeneous backbones, and extend attribution signals to vision-language interaction for alignment-sensitive tasks. We also plan to strengthen performance under extreme sampling through adaptive anchor selection and uncertainty-aware aggregation. Finally, incorporating robustness and fairness criteria would enable more reliable deployment across diverse scenarios.
\section*{Limitations}

First, our evaluation mainly relies on top-$k$ ranking diagnostics, including intersection-based NDCG@k and Kendall's $\tau$@k. While these metrics are useful for comparing the relative ordering of commonly selected top models, they do not exhaustively capture ranking errors over the full model zoo. We therefore report standard full-ranking metrics in Appendix~\ref{ap:standard_metrics} and treat the intersection-based metrics as complementary diagnostics. Moreover, our current instantiation uses layer conductance and the embedding norm as a label-free scalar probe; although this design avoids labels and text prompts and is suitable for low-supervision settings, alternative attribution methods, scalar probes, or vision--language alignment objectives may further improve task representation.

Second, conductance representations are inevitably affected by model architecture and pretraining. Although the GroupRank baseline and model-clustering analysis suggest that coarse task/domain groups or architecture identity alone cannot fully explain the observed gains, fully disentangling task-specific functional structure from architecture-induced effects remains future work. In addition, conductance extraction is more computationally expensive than a standard forward-only pass. Source-side representations can be precomputed and reused, but target-side attribution over a large model zoo remains the main practical bottleneck, motivating more efficient attribution extraction and architecture-aware caching strategies.




\bibliography{custom}

\appendix

\section{Layer Conductance's Detail}
\label{layerconduct}

\paragraph{Intuition.}
Layer conductance quantifies the contribution of individual neurons to the final prediction by measuring the flow of attribution through them. A neuron with high conductance contributes more strongly to the scalar output along the attribution path, suggesting a larger role in the corresponding forward computation.

\paragraph{Background and notation.}
Let $F:\mathbb{R}^n\rightarrow\mathbb{R}$ denote the scalar model output used for attribution, with input $x\in\mathbb{R}^n$ and baseline $x'\in\mathbb{R}^n$.
We consider the straight-line path $\gamma(\alpha)=x' + \alpha(x-x')$ for $\alpha\in[0,1]$.

\paragraph{Scalar probe used in this paper.}
For a VLM $m$, let $f_m(x)\in\mathbb{R}^{p}$ be its image embedding. We define
\begin{equation}
    F_m(x) \triangleq \|f_m(x)\|_2 ,
\end{equation}
and compute layer conductance with respect to $F_m(x)$. This serves as a simple label-free scalar probe that does not require labels, text prompts, or class semantics, and aggregates attribution across embedding dimensions.
Unless otherwise specified, we use the default settings in Captum~\cite{captum}.

\paragraph{Integrated Gradients.}
Integrated Gradients (IG) attributes the prediction to input features. For the $i$-th feature,
\begin{equation}
\mathrm{IG}_i(x) := (x_i-x_i') \int_{0}^{1}\frac{\partial F(\gamma(\alpha))}{\partial x_i}\, d\alpha .
\end{equation}

\paragraph{Layer conductance formulation.}
For a hidden neuron $y$ (scalar activation), layer conductance extends IG by decomposing attribution flow through $y$. The conductance of $y$ with respect to input feature $i$ is
\begin{equation}
\mathrm{Cond}^{y}_i(x) := (x_i-x_i') \int_{0}^{1}
\frac{\partial F(\gamma(\alpha))}{\partial y}
\frac{\partial y}{\partial x_i}\, d\alpha .
\end{equation}
The total conductance of $y$ aggregates contributions from all inputs:
\begin{equation}
\label{eq:3}
\mathrm{Cond}^{y}(x) := \sum_{i}(x_i-x_i') \int_{0}^{1}
\frac{\partial F(\gamma(\alpha))}{\partial y}
\frac{\partial y}{\partial x_i}\, d\alpha .
\end{equation}
For a layer $h(x)\in\mathbb{R}^m$ with neurons $\{y_j\}_{j=1}^{m}$, layer conductance returns $\{\mathrm{Cond}^{y_j}(x)\}_{j=1}^{m}$.
It satisfies completeness: $\sum_{j=1}^{m}\mathrm{Cond}^{y_j}(x) = F(x)-F(x')$.

\paragraph{Efficient computation.}
Direct evaluation of Eq.~\ref{eq:3} requires computing $\frac{\partial y}{\partial x_i}$ for all input dimensions, which is costly.
By recognizing the equivalence between total conductance and path integrated gradients on hidden activations, we rewrite
\begin{equation}
\mathrm{Cond}^{y}(x)
=
\int_{0}^{1}
\frac{\partial F(\gamma(\alpha))}{\partial \gamma_y(\alpha)}
\frac{\partial \gamma_y(\alpha)}{\partial \alpha}\, d\alpha ,
\end{equation}
where $\gamma_y(\alpha)=y(\gamma(\alpha))$ is the activation of $y$ along the path.
We approximate the integral via Riemann sum. Let $x^{(k)} = x' + \frac{k}{n}(x-x')$ and $y^{(k)} = y(x^{(k)})$ for $k=0,\dots,n$. Then
\begin{equation}
\label{eq:5}
\mathrm{Cond}^{y}(x)
\approx
\sum_{k=1}^{n}
\frac{\partial F(x^{(k)})}{\partial y^{(k)}}
\left(y^{(k)} - y^{(k-1)}\right).
\end{equation}
This formulation avoids computing input-level gradients and scales efficiently to high-dimensional settings.

\paragraph{Block score aggregation.}
Captum returns a tensor-valued conductance attribution for the chosen block output~\cite{captum}. We convert it to a scalar block score by taking the mean absolute attribution over all tensor entries:
\begin{equation}
    g_m^{(i)}(x) \triangleq \mathrm{Mean}\!\left(\left|\mathrm{Cond}_m^{(i)}(x)\right|\right),
\end{equation}
where $\mathrm{Cond}_m^{(i)}(x)$ denotes the conductance attribution tensor for block $i$.
We then form $g_m(x)=[g_m^{(1)}(x),\ldots,g_m^{(d_m)}(x)]^\top$, and compute the task representation by averaging $g_m(x)$ over unlabeled samples.

\paragraph{Implementation.}
We use PyTorch Captum ({captum.attr.LayerConductance}), which implements the efficient computation in Eq.~\ref{eq:5}.
For detailed derivations and theoretical properties, see \citep{layerconductance1,captum}.

\section{Details regarding Section~\ref{sec:qualitative}}
\label{ap:qi_details}
We analyze five representative VLMs (ResNet and ViT backbones; diverse pre-training sources) on the same $|\mathcal{T}|{=}21$ downstream tasks. 
For each task $\tau$, we sample $N_{\text{src}}{=}25$ unlabeled images and compute conductance-based task representations $\hat{v}_{m,\tau}$ following Section~\ref{sec:preliminary}. We use $\eta{=}2.5$ and $\epsilon{=}10^{-8}$ throughout.

\paragraph{Performance gap matrix (Figure~\ref{fig:easy_example}, left).}
For each model $m$, let $a_m(\tau)$ be its benchmark score on task $\tau$ (e.g., accuracy; for non-accuracy metrics we use the task's primary scalar score). We define the ground-truth task-pair performance gap matrix
\begin{equation}
G_m(\tau,\sigma)\ \triangleq\ \big|a_m(\tau)-a_m(\sigma)\big|.
\label{eq:app_gt_gap}
\end{equation}

\paragraph{Proxy gap matrix (Figure~\ref{fig:easy_example}, right).}
\emph{Semantic:} we embed each task $\tau$ using model $m$'s text encoder (averaging normalized text features over class-name prompts \texttt{`A photo of \{c\}.'}) to obtain $e_m(\tau)$, and set
\begin{equation}
G^{\text{sem}}_m(\tau,\sigma)\ \triangleq\ 1-\cos\!\big(e_m(\tau),e_m(\sigma)\big).
\label{eq:app_sem_gap}
\end{equation}
\emph{Conductance:} we compute the target-conditioned importance $\alpha_{m,\tau}$ (Section~\ref{sec:importance_distribution}) and directional divergence (Section~\ref{sec:dcd})
\begin{equation}
D_m(\tau\!\to\!\sigma)\ \triangleq\ \sum_{i=1}^{d_m}\alpha_{m,\tau}^{(i)}\cdot
\frac{\big|\hat{v}_{m,\tau}^{(i)}-\hat{v}_{m,\sigma}^{(i)}\big|}{\big|\hat{v}_{m,\tau}^{(i)}\big|+\epsilon},
\label{eq:app_dcd}
\end{equation}
then symmetrize it to match the undirected proxy gap in Figure~\ref{fig:easy_example}:
\begin{equation}
G^{\text{cond}}_m(\tau,\sigma)\ \triangleq\ \tfrac{1}{2}\!\left(D_m(\tau\!\to\!\sigma)+D_m(\sigma\!\to\!\tau)\right).
\label{eq:app_cond_gap}
\end{equation}

\paragraph{Proxy reliability (Table~\ref{tab:correlation_comparison}).}
Let $\mathrm{Vec}(\cdot)$ flatten the upper-triangular entries excluding the diagonal (Figure~\ref{fig:easy_example}). For each model $m$, we compute
\begin{equation}
\begin{aligned}
\rho_{\text{cond}}(m)
&= \mathrm{Spearman}\!\left(\mathrm{Vec}(G_m),\mathrm{Vec}(G^{\text{cond}}_m)\right), \\
\rho_{\text{sem}}(m)
&= \mathrm{Spearman}\!\left(\mathrm{Vec}(G_m),\mathrm{Vec}(G^{\text{sem}}_m)\right).
\end{aligned}
\label{eq:app_spearman}
\end{equation}

\paragraph{Model-specific task perception (Figure~\ref{fig:conductance_heatmap}).}
For any two models $(m_1,m_2)$, we measure agreement in conductance-induced task distances by
\begin{equation}
\begin{aligned}
\mathrm{Corr}(m_1,m_2)
&= \mathrm{Spearman}\!\left(
\mathrm{Vec}(G^{\text{cond}}_{m_1}), \right. \\
&\qquad\left. \mathrm{Vec}(G^{\text{cond}}_{m_2})\right).
\end{aligned}
\label{eq:app_modelcorr}
\end{equation}
and stack $\mathrm{Corr}(m_1,m_2)$ into the model--model correlation matrix.

\section{Theoretical Analysis and Derivations}
\label{sec:theory}

Our method adopts model-specific task representations and a directional transfer distance, contrasting with prior symmetric approaches~\citep{lovm_modelgpt,swab}. This appendix provides the theoretical motivation for asymmetry under a target-conditioned coverage view of transfer, derives the entropy-regularized target-conditioned importance distribution, and proves that Directional Conductance Divergence (DCD) is a smooth relaxation of an idealized set-restricted formulation.

\subsection{Coverage Semantics and Asymmetry}

\begin{definition}[Salient Set]
\label{def:salient-set}
Let $\mathbf{u}_{m,\tau} \in \mathbb{R}^{d_m}$ be the normalized representation of task $\tau$ on model $m$. Let $\rho_k(\mathbf{u}_{m,\tau})$ denote the $k$-th largest value in $\mathbf{u}_{m,\tau}$. The \emph{salient set} of task $\tau$ is defined as
\begin{equation}
    \mathcal{S}_{m,\tau}^{(k)} \triangleq 
    \left\{ i \in \{1,\dots,d_m\} \;\middle|\; 
    u_{m,\tau}^{(i)} \ge \rho_k(\mathbf{u}_{m,\tau}) \right\}. 
\end{equation}
This set represents the top-$k$ functional blocks on which task $\tau$ relies most heavily; for an appropriate choice of $k$, it serves as a proxy for the task-critical blocks.
\end{definition}

\paragraph{Coverage Semantics for Model Selection.}
We argue that transfer distance from a source task $\sigma$ to a target task $\tau$ should follow a target-conditioned \emph{coverage semantics}. In our setting, coverage does not merely mean that the source activates a superset of target-critical blocks. Instead, because DCD compares conductance magnitudes, a source is considered useful when it matches or closely approximates the target on the target-critical functional blocks. Discrepancies outside the target salient set should not be heavily penalized, since they correspond to source-specific mechanisms that are not essential for the target.

This target-conditioned view naturally induces asymmetry: a source task may sufficiently cover the blocks required by the target, while the target may fail to cover additional blocks that are salient for the source. We formalize this principle via two assumptions on a directional distance function $d(\cdot\to\cdot)$.

\begin{assumption}[Target-Sufficiency]
\label{ass:target-sufficiency}
If the source task $\sigma$ matches the target task $\tau$ on all indices in the target salient set in terms of normalized conductance magnitude, the distance should be zero regardless of discrepancies outside the target salient set:
\[
    \forall i \in \mathcal{S}_{m,\tau}^{(k)}, \quad 
    u_{m,\sigma}^{(i)} = u_{m,\tau}^{(i)}
    \implies 
    d(\tau \to \sigma) = 0.  
\]
\end{assumption}

\begin{assumption}[Salient-Discriminativity]
\label{ass:salient-discriminativity}
Conversely, if the direction is reversed and the target task $\tau$ fails to match the source task $\sigma$ on some block that is salient for $\sigma$, then the reverse distance must be strictly positive:
\[
    \exists j \in \mathcal{S}_{m,\sigma}^{(k)} 
    \text{ s.t. } 
    u_{m,\tau}^{(j)} \neq u_{m,\sigma}^{(j)}
    \implies 
    d(\sigma \to \tau) > 0.
\]
\end{assumption}

\begin{proposition}[Impossibility of Symmetry]
\label{prop:asymmetry}
Let $d(\cdot \to \cdot)$ satisfy Assumptions~\ref{ass:target-sufficiency} and~\ref{ass:salient-discriminativity}. Suppose $\mathcal{S}_{m,\tau}^{(k)} \subsetneq \mathcal{S}_{m,\sigma}^{(k)}$, and the representations match on $\mathcal{S}_{m,\tau}^{(k)}$ but differ on at least one residual index in $\mathcal{S}_{m,\sigma}^{(k)} \setminus \mathcal{S}_{m,\tau}^{(k)}$. Then $d$ is necessarily asymmetric.
\end{proposition}

\begin{proof}
Evaluate the distance in both directions:
\begin{itemize}
    \item \textbf{Forward ($\tau \to \sigma$).} Since $\sigma$ matches $\tau$ on all indices in $\mathcal{S}_{m,\tau}^{(k)}$, Assumption~\ref{ass:target-sufficiency} yields $d(\tau \to \sigma) = 0$.
    \item \textbf{Reverse ($\sigma \to \tau$).} Since $\mathcal{S}_{m,\tau}^{(k)} \subsetneq \mathcal{S}_{m,\sigma}^{(k)}$, there exists at least one index $j \in \mathcal{S}_{m,\sigma}^{(k)} \setminus \mathcal{S}_{m,\tau}^{(k)}$ where the representations differ. By Assumption~\ref{ass:salient-discriminativity}, $d(\sigma \to \tau) > 0$.
\end{itemize}
Hence $d(\tau \to \sigma) \neq d(\sigma \to \tau)$, so $d$ is asymmetric.
\end{proof}

\begin{remark}
\label{rem:symmetric-metrics}
Proposition~\ref{prop:asymmetry} reveals a limitation of symmetric metrics, such as cosine similarity and Euclidean distance, under target-conditioned coverage semantics. A symmetric metric cannot distinguish whether the unmatched blocks are irrelevant source-specific mechanisms or target-critical mechanisms. As a result, symmetric metrics may over-penalize comprehensive source tasks that contain additional mechanisms not required by the target.
\end{remark}

\subsection{Entropy-Regularized Alignment}
\label{ap:importance_weight}

In this section, we provide the detailed derivation for the closed-form solution of the importance distribution $\alpha_{m,\tau}$. For notational simplicity and without loss of generality, we drop the subscripts $(m,\tau)$ and denote the alignment vector as $u \in \mathbb{R}^d$ and the decision variable as $\alpha \in \Delta^{d-1}$.

\paragraph{Problem Formulation.}
Consider the entropy-regularized alignment problem:
\begin{equation}
\label{eq:ap_obj}
    \max_{\alpha \in \Delta^{d-1}} \left\{
        \langle \alpha, u \rangle + \frac{1}{\eta} H(\alpha)
    \right\},
\end{equation}
where $\Delta^{d-1} = \{\alpha \in \mathbb{R}^d : \alpha \geq 0,\, \mathbf{1}^\top \alpha = 1\}$ denotes the $(d-1)$-dimensional probability simplex, $H(\alpha) = -\sum_i \alpha_i \log \alpha_i$ is the Shannon entropy, and $\eta > 0$ is a temperature parameter.

\paragraph{Closed-Form Solution.}
\begin{theorem}
\label{thm:softmax}
The optimization problem in~\eqref{eq:ap_obj} admits a unique global maximizer $\alpha^*$, given by:
\begin{equation}
\label{eq:ap_solution}
\begin{aligned}
    \alpha^* &= \mathrm{softmax}(\eta u), \\
    \alpha_i^* &= \frac{\exp(\eta u_i)}{\sum_{j=1}^{d} \exp(\eta u_j)}.
\end{aligned}
\end{equation}
\end{theorem}

\begin{proof}
The proof proceeds in three steps.

\paragraph{Step 1: Strict Concavity and Interior Solution.}
The objective $\mathcal{J}(\alpha)$ is strictly concave over $\mathrm{int}(\Delta^{d-1})$, being the sum of a linear function and the strictly concave entropy $H(\alpha)$. This guarantees uniqueness of the maximizer. Moreover, since
\[
    \frac{\partial H(\alpha)}{\partial \alpha_i}
    =
    -(\log \alpha_i + 1)
    \to +\infty
    \quad \text{as } \alpha_i \to 0^+,
\]
the maximizer cannot occur on the boundary of the simplex. Hence the optimal solution lies strictly in the simplex interior, i.e., $\alpha_i^*>0$ for all $i$, and the non-negativity constraints are inactive.

\paragraph{Step 2: KKT Conditions.}
We form the Lagrangian with multiplier $\lambda$ for the equality constraint:
\begin{equation}
    \mathcal{L}(\alpha, \lambda) 
    = 
    \langle \alpha, u \rangle 
    - \frac{1}{\eta} \sum_i \alpha_i \log \alpha_i 
    + \lambda (\mathbf{1}^\top \alpha - 1).
\end{equation}
The stationarity condition $\nabla_\alpha \mathcal{L} = 0$ yields:
\begin{equation}
\label{eq:foc}
    u_i - \frac{1}{\eta}(\log \alpha_i + 1) + \lambda = 0, 
    \quad \forall i.
\end{equation}
Rearranging~\eqref{eq:foc} gives
\begin{equation}
    \log \alpha_i = \eta(u_i + \lambda) - 1
    \quad \Longrightarrow \quad
    \alpha_i = e^{\eta \lambda - 1} \cdot e^{\eta u_i}.
\end{equation}
Denoting $C \triangleq e^{\eta \lambda - 1}$, we have $\alpha_i = C \cdot e^{\eta u_i}$.

\paragraph{Step 3: Normalization.}
Imposing $\mathbf{1}^\top \alpha = 1$:
\begin{equation}
    C \sum_{j=1}^{d} e^{\eta u_j} = 1
    \quad \Longrightarrow \quad
    C = \left( \sum_{j=1}^{d} e^{\eta u_j} \right)^{-1}.
\end{equation}
Substituting back yields~\eqref{eq:ap_solution}. Since $\exp(\eta u_i) > 0$ for all $i$, the solution satisfies $\alpha_i^* > 0$, confirming feasibility.
\end{proof}

\subsection{DCD as a Continuous Relaxation}
\label{ap:dcd_relax}

The target-conditioned coverage semantics above can be instantiated by a set-restricted divergence that focuses solely on the target's salient blocks. Recall that $\mathcal{S}_{m,\tau}^{(k)}$ denotes the salient set of task $\tau$ on model $m$ (Definition~\ref{def:salient-set}). Define the tail mass of the importance distribution as
\begin{equation}
    t_{m,\tau}^{(k)}(\eta) 
    \triangleq 
    \sum_{i\notin \mathcal{S}_{m,\tau}^{(k)}} 
    \alpha_{m,\tau}^{(i)}.
\end{equation}
The hard importance distribution supported exclusively on $\mathcal{S}_{m,\tau}^{(k)}$ is then
\begin{equation}
\bar{\alpha}_{m,\tau}^{(i)}(\eta) \triangleq 
\begin{cases} 
\alpha_{m,\tau}^{(i)} / (1-t_{m,\tau}^{(k)}(\eta)), 
& i \in \mathcal{S}_{m,\tau}^{(k)}, \\ 
0, 
& \text{otherwise}, 
\end{cases}
\end{equation}
which yields the \emph{set-restricted divergence}
\begin{equation}
    d_{m}^{(k)}(\tau\to\sigma;\eta) 
    \triangleq 
    \sum_{i=1}^{d_m} 
    \bar{\alpha}_{m,\tau}^{(i)}(\eta)\,
    \delta_m^{(i)}(\tau,\sigma).
\end{equation}
Here, $d_m^{(k)}$ focuses only on target-salient blocks and ignores discrepancies outside $\mathcal{S}_{m,\tau}^{(k)}$. It therefore exactly captures the idealized target-sufficiency principle in Assumption~\ref{ass:target-sufficiency}. However, it requires a discrete choice of $k$, for which no universally principled criterion exists. We therefore adopt DCD $D_m(\tau\to\sigma)$ as a smooth relaxation, replacing the hard indicator $\mathbf{1}[i\in \mathcal{S}_{m,\tau}^{(k)}]$ with softmax weights $\alpha_{m,\tau}$. We quantify this relaxation via the tail mass $t_{m,\tau}^{(k)}(\eta)$.

\begin{lemma}[Tail Mass Bound]
\label{lem:concentration}
Let $u_{(1)}\ge \cdots \ge u_{(d_m)}$ be the sorted entries of $\mathbf{u}_{m,\tau}$. Assume $\Delta_k \triangleq u_{(k)}-u_{(k+1)} > 0$. Then the tail mass satisfies
\begin{equation}
    t_{m,\tau}^{(k)}(\eta) 
    \le 
    \frac{d_m-k}{k}\exp(-\eta \Delta_k).
\end{equation}
\end{lemma}

\begin{proof}
Let $\mathcal{S}\triangleq \mathcal{S}_{m,\tau}^{(k)}$ and denote $\alpha^{(i)}\triangleq \alpha_{m,\tau}^{(i)}$ for brevity. Since $\Delta_k>0$, the $k$-th and $(k+1)$-th largest entries are separated, and hence $|\mathcal{S}|=k$. By definition of softmax,
\begin{equation}
\begin{aligned}
    \sum_{i\notin \mathcal{S}}\alpha^{(i)}
    &=
    \sum_{i\notin \mathcal{S}}
    \frac{\exp(\eta u_{m,\tau}^{(i)})}
    {\sum_{j=1}^{d_m}\exp(\eta u_{m,\tau}^{(j)})} \\
    &=
    \frac{
    \sum_{i\notin \mathcal{S}}\exp(\eta u_{m,\tau}^{(i)})
    }{
    \sum_{j=1}^{d_m}\exp(\eta u_{m,\tau}^{(j)})
    }.
\end{aligned}
\end{equation}
Lower bound the denominator by keeping only the contribution of indices in $\mathcal{S}$:
\begin{equation}
    \sum_{j=1}^{d_m}\exp(\eta u_{m,\tau}^{(j)})
    \ge
    \sum_{j\in \mathcal{S}}\exp(\eta u_{m,\tau}^{(j)}).
\end{equation}
Hence
\begin{equation}
    \sum_{i\notin \mathcal{S}}\alpha^{(i)}
    \le
    \frac{
    \sum_{i\notin \mathcal{S}}\exp(\eta u_{m,\tau}^{(i)})
    }{
    \sum_{j\in \mathcal{S}}\exp(\eta u_{m,\tau}^{(j)})
    }.
\end{equation}
For any $i\notin \mathcal{S}$, $u_{m,\tau}^{(i)}\le u_{(k+1)}$, while for any $j\in \mathcal{S}$, $u_{m,\tau}^{(j)}\ge u_{(k)}$. Therefore,
\begin{equation}
\begin{aligned}
    \sum_{i\notin \mathcal{S}}\exp(\eta u_{m,\tau}^{(i)})
    &\le 
    (d_m-k)\exp(\eta u_{(k+1)}), \\
    \sum_{j\in \mathcal{S}}\exp(\eta u_{m,\tau}^{(j)})
    &\ge 
    k\exp(\eta u_{(k)}).
\end{aligned}
\end{equation}
Combining the above inequalities yields
\begin{equation}
\begin{aligned}
    \sum_{i\notin \mathcal{S}}\alpha^{(i)}
    &\le 
    \frac{(d_m-k)\exp(\eta u_{(k+1)})}
    {k\exp(\eta u_{(k)})} \\
    &=
    \frac{d_m-k}{k}
    \exp\!\big(-\eta(u_{(k)}-u_{(k+1)})\big) \\
    &=
    \frac{d_m-k}{k}\exp(-\eta\Delta_k).
\end{aligned}
\end{equation}
This concludes the proof of Lemma~\ref{lem:concentration}.
\end{proof}

\begin{proposition}[DCD as a Set-Restricted Relaxation]
\label{prop:dcd-relaxation}
Assume that conductance representations are uniformly bounded, i.e., there exists $V_{\max}>0$ such that $|v_{m,\tau}^{(i)}|\le V_{\max}$ for all $m,\tau,i$. Since the denominator in $\delta_m^{(i)}(\tau,\sigma)$ is lower bounded by $\epsilon$, there exists $B=2V_{\max}/\epsilon$ such that $\delta_m^{(i)}(\tau,\sigma) \le B$ for all $i$. For any source task $\sigma$, the following holds:
\begin{align}
&   
D_m(\tau\to\sigma) 
\nonumber \\
& =
    \left(1-t_{m,\tau}^{(k)}(\eta) \right)  
    d_{m}^{(k)}(\tau\to\sigma;\eta)
    + r_{m}^{(k)}(\tau,\sigma;\eta), 
\end{align}
where
\begin{equation}
    r_{m}^{(k)}(\tau,\sigma;\eta)
    \triangleq
    \sum_{i\notin \mathcal{S}_{m,\tau}^{(k)}} 
    \alpha_{m,\tau}^{(i)}\,
    \delta_m^{(i)}(\tau,\sigma).  
\end{equation}
Furthermore,
\begin{equation}
\begin{aligned}
0 \le r_{m}^{(k)}(\tau,\sigma;\eta) 
&\le 
B\, t_{m,\tau}^{(k)}(\eta), \\
\bigl|D_m(\tau\to\sigma)-d_{m}^{(k)}(\tau\to\sigma;\eta)\bigr|
&\le 
2B\, t_{m,\tau}^{(k)}(\eta).
\end{aligned}
\end{equation}
Consequently, combining with Lemma~\ref{lem:concentration}, $D_m(\tau\to\sigma)$ approximates the set-restricted divergence $d_m^{(k)}(\tau\to\sigma;\eta)$ with an $O(\exp(-\eta\Delta_k))$ error.
\end{proposition}

\begin{proof}
Let $\mathcal{S}\triangleq \mathcal{S}_{m,\tau}^{(k)}$, $\alpha^{(i)}\triangleq \alpha_{m,\tau}^{(i)}$, and $\delta^{(i)}\triangleq \delta_m^{(i)}(\tau,\sigma)$ for brevity. Define the tail mass $t\triangleq \sum_{i\notin \mathcal{S}}\alpha^{(i)}$, so that $\sum_{i\in \mathcal{S}}\alpha^{(i)}=1-t$. By definition,
\begin{equation}
\begin{aligned}
    D_m(\tau\to\sigma)
    &=
    \sum_{i=1}^{d_m}\alpha^{(i)}\delta^{(i)} \\
    &=
    \sum_{i\in \mathcal{S}}\alpha^{(i)}\delta^{(i)}
    +
    \sum_{i\notin \mathcal{S}}\alpha^{(i)}\delta^{(i)}.
\end{aligned}
\end{equation}
Define the hard distribution $\bar{\alpha}^{(i)}$ supported on $\mathcal{S}$ by
\begin{equation}
    \bar{\alpha}^{(i)}
    =
    \begin{cases}
        \alpha^{(i)}/(1-t), 
        & i\in \mathcal{S},\\
        0, 
        & i\notin \mathcal{S}.
    \end{cases}
\end{equation}
Then the set-restricted divergence is
\begin{equation}
\begin{aligned}
    d_m^{(k)}(\tau\to\sigma;\eta)
    &=
    \sum_{i=1}^{d_m}\bar{\alpha}^{(i)}\delta^{(i)} \\
    &=
    \sum_{i\in \mathcal{S}}
    \frac{\alpha^{(i)}}{1-t}\delta^{(i)}.
\end{aligned}
\end{equation}
Thus,
\begin{equation}
\begin{aligned}
    \sum_{i\in \mathcal{S}}\alpha^{(i)}\delta^{(i)}
    &=
    (1-t)
    \sum_{i\in \mathcal{S}}
    \frac{\alpha^{(i)}}{1-t}\delta^{(i)} \\
    &=
    (1-t)\,d_m^{(k)}(\tau\to\sigma;\eta).
\end{aligned}
\end{equation}
Let
\begin{equation}
    r_m^{(k)}(\tau,\sigma;\eta)
    \triangleq
    \sum_{i\notin \mathcal{S}}\alpha^{(i)}\delta^{(i)}.
\end{equation}
This gives the claimed decomposition:
\begin{equation}
\begin{aligned}
    D_m(\tau\to\sigma)
    &= (1-t)\,d_m^{(k)}(\tau\to\sigma;\eta) \\
    &\quad + r_m^{(k)}(\tau,\sigma;\eta).
\end{aligned}
\end{equation}
Since each $\alpha^{(i)}\ge 0$ and each $\delta^{(i)}\ge 0$, we have $r_m^{(k)}(\tau,\sigma;\eta)\ge 0$. Under the boundedness condition $\delta^{(i)}\le B$ for all $i$,
\begin{equation}
\begin{aligned}
    r_m^{(k)}(\tau,\sigma;\eta)
    &=
    \sum_{i\notin \mathcal{S}}\alpha^{(i)}\delta^{(i)} \\
    &\le
    \sum_{i\notin \mathcal{S}}\alpha^{(i)}\cdot B \\
    &=
    Bt.
\end{aligned}
\end{equation}
This proves $0\le r_m^{(k)}(\tau,\sigma;\eta)\le Bt$. Moreover, since $d_m^{(k)}(\tau\to\sigma;\eta)$ is an expectation of $\delta^{(i)}$ under a probability distribution supported on $\mathcal{S}$, it also satisfies $0\le d_m^{(k)}(\tau\to\sigma;\eta)\le B$. Therefore,
\begin{equation}
\begin{aligned}
    &D_m(\tau\to\sigma)
    - d_m^{(k)}(\tau\to\sigma;\eta) \\
    &\quad =
    -t\, d_m^{(k)}(\tau\to\sigma;\eta)
    + r_m^{(k)}(\tau,\sigma;\eta).
\end{aligned}
\end{equation}
Taking absolute values gives
\begin{equation}
\begin{aligned}
    &\big|D_m(\tau\to\sigma)
    - d_m^{(k)}(\tau\to\sigma;\eta)\big| \\
    &\quad\le tB + Bt
    = 2Bt.
\end{aligned}
\end{equation}
Finally, substituting the bound on $t$ from Lemma~\ref{lem:concentration} yields the stated $O(\exp(-\eta\Delta_k))$ approximation rate.
\end{proof}

\begin{corollary}
\label{cor:dcd-coverage}
Assume $\Delta_k>0$ and the boundedness condition in Proposition~\ref{prop:dcd-relaxation}. If the source matches the target on the target salient set, i.e.,
\[
    \delta_m^{(i)}(\tau,\sigma)=0 
    \quad 
    \text{for all } 
    i\in \mathcal{S}_{m,\tau}^{(k)},
\]
then
\begin{equation}
\begin{aligned}
D_m(\tau\to\sigma)
&=
r_{m}^{(k)}(\tau,\sigma;\eta) \\
&\le 
B\, t_{m,\tau}^{(k)}(\eta) \\
&\le 
B\cdot\frac{d_m-k}{k}\exp(-\eta\Delta_k).
\end{aligned}
\end{equation}
Hence $D_m(\tau\to\sigma)\to 0$ when $\eta$ or $\Delta_k$ is sufficiently large. Thus, DCD recovers the discrete Target-Sufficiency semantics in Assumption~\ref{ass:target-sufficiency} up to an exponentially small relaxation error, while remaining intrinsically asymmetric as required by Proposition~\ref{prop:asymmetry}, through its target-conditioned weighting $\alpha_{m,\tau}$.
\end{corollary}
\section{Dataset and Model Details}
\label{ap:modelanddataset}
\subsection{Model Zoo}
We curated a diverse model zoo by extending the collections in LOVM~\cite{lovm_modelgpt} and SWAB~\cite{swab}. Incorporating additional architectures like EVA01-g-14~\cite{eva} and nllb-clip-base~\cite{nlibclip}, our final benchmark comprises 48 models, as detailed in Table~\ref{tab:model_zoo}.

\begin{table*}[!htb]
\centering
\scriptsize
\setlength{\tabcolsep}{3pt}
\renewcommand{\arraystretch}{1.1}
\begin{tabular*}{\textwidth}{@{\extracolsep{\fill}} p{0.29\textwidth} p{0.31\textwidth} rrr}
\toprule
Model (Family) & Pretrain Tag & Params(M) & FLOPs(B) & ImageNet Acc.(\%) \\
\midrule
\url{RN50} & \url{openai} & 102.01 & 18.18 & 59.82 \\
\url{RN50} & \url{cc12m} & 102.01 & 18.18 & 35.91 \\
\url{RN101} & \url{openai} & 119.69 & 25.50 & 62.28 \\
\url{RN101} & \url{yfcc15m} & 119.69 & 25.50 & 34.07 \\
\url{RN101-quickgelu} & \url{openai} & 119.69 & 25.50 & 62.28 \\
\url{RN101-quickgelu} & \url{yfcc15m} & 119.69 & 25.50 & 34.87 \\
\url{RN50x4} & \url{openai} & 178.30 & 51.82 & 66.27 \\
\url{RN50x64} & \url{openai} & 623.26 & 552.65 & 73.91 \\
\url{ViT-B-32} & \url{openai} & 151.28 & 14.78 & 63.32 \\
\url{ViT-B-32} & \url{laion2b_e16} & 151.28 & 14.78 & 65.65 \\
\url{ViT-B-32} & \url{datacomp_xl_s13b_b90k} & 151.28 & 14.78 & 69.17 \\
\url{ViT-B-32} & \url{commonpool_m_clip_s128m_b4k} & 151.28 & 14.78 & 27.25 \\
\url{ViT-B-32-256} & \url{datacomp_s34b_b86k} & 151.29 & 17.46 & 72.81 \\
\url{ViT-B-32-quickgelu} & \url{laion400m_e31} & 151.28 & 14.78 & 62.94 \\
\url{ViT-B-32-quickgelu} & \url{metaclip_fullcc} & 151.28 & 14.78 & 67.66 \\
\url{ViT-B-16} & \url{openai} & 149.62 & 41.09 & 68.34 \\
\url{ViT-B-16} & \url{laion2b_s34b_b88k} & 149.62 & 41.09 & 70.23 \\
\url{ViT-B-16} & \url{datacomp_l_s1b_b8k} & 149.62 & 41.09 & 63.10 \\
\url{ViT-B-16} & \url{commonpool_l_laion_s1b_b8k} & 149.62 & 41.09 & 55.26 \\
\url{ViT-B-16} & \url{dfn2b} & 149.62 & 41.09 & 76.24 \\
\url{ViT-B-16-quickgelu} & \url{metaclip_fullcc} & 149.62 & 41.09 & 72.12 \\
\url{ViT-B-16-plus-240} & \url{laion400m_e31} & 208.38 & 64.03 & 69.04 \\
\url{ViT-L-14} & \url{openai} & 427.62 & 175.33 & 75.54 \\
\url{ViT-L-14} & \url{laion400m_e31} & 427.62 & 175.33 & 72.71 \\
\url{ViT-L-14} & \url{datacomp_xl_s13b_b90k} & 427.62 & 175.33 & 79.21 \\
\url{ViT-L-14} & \url{commonpool_xl_clip_s13b_b90k} & 427.62 & 175.33 & 76.37 \\
\url{ViT-L-14-quickgelu} & \url{metaclip_fullcc} & 427.62 & 175.33 & 79.17 \\
\url{ViT-L-14-quickgelu} & \url{dfn2b} & 427.62 & 175.33 & 81.41 \\
\url{ViT-L-14-336} & \url{openai} & 427.94 & 395.22 & 76.56 \\
\url{ViT-H-14} & \url{laion2b_s32b_b79k} & 986.11 & 381.68 & 77.96 \\
\url{ViT-H-14-quickgelu} & \url{metaclip_fullcc} & 986.11 & 381.68 & 80.51 \\
\url{ViT-H-14-378-quickgelu} & \url{dfn5b} & 986.71 & 1054.05 & 84.37 \\
\url{ViT-g-14} & \url{laion2b_s12b_b42k} & 1366.68 & 581.15 & 76.63 \\
\url{ViT-bigG-14} & \url{laion2b_s39b_b160k} & 2539.57 & 1065.36 & 80.09 \\
\url{roberta-ViT-B-32} & \url{laion2b_s12b_b32k} & 212.72 & 105.87 & 61.71 \\
\url{xlm-roberta-base-ViT-B-32} & \url{laion5b_s13b_b90k} & 366.12 & 105.87 & 62.36 \\
\url{convnext_base_w} & \url{laion2b_s13b_b82k} & 179.39 & 49.38 & 70.78 \\
\url{convnext_base_w_320} & \url{laion_aesthetic_s13b_b82k} & 179.39 & 71.94 & 71.67 \\
\url{convnext_large_d} & \url{laion2b_s26b_b102k_augreg} & 351.77 & 107.50 & 75.91 \\
\url{convnext_large_d_320} & \url{laion2b_s29b_b131k_ft} & 351.77 & 157.98 & 76.60 \\
\url{convnext_xxlarge} & \url{laion2b_s34b_b82k_augreg_soup} & 1200.58 & 443.03 & 79.47 \\
\url{coca_ViT-B-32} & \url{laion2b_s13b_b90k} & 253.56 & 33.34 & 63.31 \\
\url{coca_ViT-L-14} & \url{laion2b_s13b_b90k} & 638.45 & 214.52 & 75.61 \\
\url{EVA01-g-14} & \url{laion400m_s11b_b41k} & 1136.44 & 547.36 & 78.52 \\
\url{EVA02-B-16} & \url{merged2b_s8b_b131k} & 149.69 & 41.09 & 74.72 \\
\url{EVA02-L-14-336} & \url{merged2b_s6b_b61k} & 428.08 & 395.16 & 80.39 \\
\url{nllb-clip-base} & \url{v1} & 501.89 & 369.60 & 24.32 \\
\url{nllb-clip-base-siglip} & \url{v1} & 507.47 & 472.91 & 39.09 \\
\bottomrule
\end{tabular*}
\caption{Overview of the model zoo used in our experiments. The table reports the model name, pre-training tag, parameter count, FLOPs, and ImageNet classification accuracy. Most data are sourced from~\cite{tan2025vision}.}
\label{tab:model_zoo}
\end{table*}

\subsection{Datasets}
To evaluate model selection under diverse real-world conditions, we employ datasets spanning various image domains and tasks, as summarized in Table~\ref{tab:datasets}.

\begin{table*}[t]
\centering
\small
\setlength{\tabcolsep}{6pt}
\renewcommand{\arraystretch}{1.1}
\begin{tabular*}{\textwidth}{@{\extracolsep{\fill}} l l p{0.42\textwidth}}
\toprule
Dataset & Domain & Task \\
\midrule
Stanford Cars~\cite{stan_car} & car picture & image classification \\
CIFAR100~\cite{cifar} & natural picture & image classification \\
CLEVR-D~\cite{clevr} & natural picture & object distance estimation \\
Country211~\cite{vlm—no1} & natural picture & geo-localization \\
Retinopathy~\cite{diabetic-retinopathy-detection} & retina scan & image classification \\
DMLab~\cite{dmlab} & natural picture & object distance estimation \\
DTD~\cite{dtd} & texture picture & image classification \\
EuroSAT~\cite{eurosat} & satellite images & image classification \\
FER2013~\cite{fer2013} & facial picture & facial expression classification \\
GTSRB~\cite{GTSRB} & traffic picture & image classification \\
KITTI~\cite{kitti} & nature picture & object distance estimation \\
MNIST~\cite{mnist} & digit picture & image classification \\
Flowers102~\cite{flower} & flower picture & image classification \\
Oxford Pet~\cite{pet} & pet photograph & image classification \\
PCam~\cite{pcam} & medical picture & image classification \\
RenderedSST2~\cite{vlm—no1} & text picture & optical character recognition \\
RESISC45~\cite{resis} & satellite picture & land cover classification \\
STL10~\cite{stl10} & natural picture & image classification \\
SUN397~\cite{sun} & natural picture & scene classification \\
SVHN~\cite{svhn} & natural picture & OCR \\
VOC2007~\cite{voc} & natural picture & image and pose classification \\
\bottomrule
\end{tabular*}
\caption{Datasets used in our evaluation, with the corresponding domain type and task. Most data are sourced from~\cite{lovm_modelgpt}.}
\label{tab:datasets}
\end{table*}

\section{Additional Details of Ablation Variants}
\label{ap:ablation_details}

\paragraph{Similarity formulation variants.}
For the similarity formulation ablation, all variants use the same conductance-based task representation and rank aggregation pipeline, but differ in how source--target task similarity is computed. 
Cosine similarity measures the normalized inner product between task representation vectors. 
Soft-KL/JSD first applies softmax normalization to task representations and then computes the Jensen--Shannon divergence as a symmetric distributional distance. 
Both variants replace the directional DCD similarity with symmetric alternatives.

\paragraph{Block-weighting variants.}
For the block-weighting ablation, all variants use the same directional relative-deviation distance:
\begin{equation}
D_m(\tau\to\sigma)
=
\sum_i \alpha_{m,\tau}^{(i)}
\frac{|v_{m,\tau}^{(i)}-v_{m,\sigma}^{(i)}|}
{|v_{m,\tau}^{(i)}|+\epsilon}.
\end{equation}
They differ only in the choice of the block-importance weights $\alpha_{m,\tau}$. 
Uniform weighting assigns equal importance to all blocks, i.e., $\alpha_{m,\tau}^{(i)}=1/d_m$. 
Target-norm weighting sets $\alpha_{m,\tau}^{(i)} \propto |v_{m,\tau}^{(i)}|$, directly using the magnitude of target conductance as block importance. 
DCD uses the entropy-regularized target-conditioned distribution 
$\alpha_{m,\tau}^{(i)}=\mathrm{softmax}(\eta u_{m,\tau}^{(i)})$, where $u_{m,\tau}$ is the normalized target representation.

\section{Additional Evaluation Results}
\label{ap:additional_results}
\subsection{Standard Ranking Metrics}
\label{ap:standard_metrics}
In addition to the metrics used in the main experiments, Table~\ref{tab:standard_metrics} reports standard ranking metrics over the full candidate model set. 
NDCG@5 (no intersection) evaluates the predicted top-5 list without restricting evaluation to the intersection between predicted and ground-truth top-5 models. 
Recall@1 measures whether the top-ranked predicted model matches the ground-truth best model.

\begin{table}[!ht]
  \centering
  \small
  \begin{tabular*}{\columnwidth}{@{\extracolsep{\fill}} lcc}
    \toprule
    \textbf{Method} & NDCG@5 (no intersection) $\uparrow$ & Recall@1 $\uparrow$ \\
    \midrule
    ModelGPT & 0.826 & 0.333 \\
    SWAB     & 0.828 & 0.381 \\
    Ours     & \textbf{0.833} & \textbf{0.476} \\
    \bottomrule
  \end{tabular*}
  \caption{\textbf{Standard ranking metrics.}
NDCG@5 (no intersection) evaluates the predicted top-5 list over the full candidate model set, while Recall@1 evaluates whether the top-ranked predicted model matches the ground-truth best model.}
  \label{tab:standard_metrics}
\end{table}

\subsection{Hyperparameter Sensitivity}
\begin{figure*}[t]
  \centering
  \begin{minipage}[t]{0.32\textwidth}
    \centering
    \includegraphics[width=\linewidth]{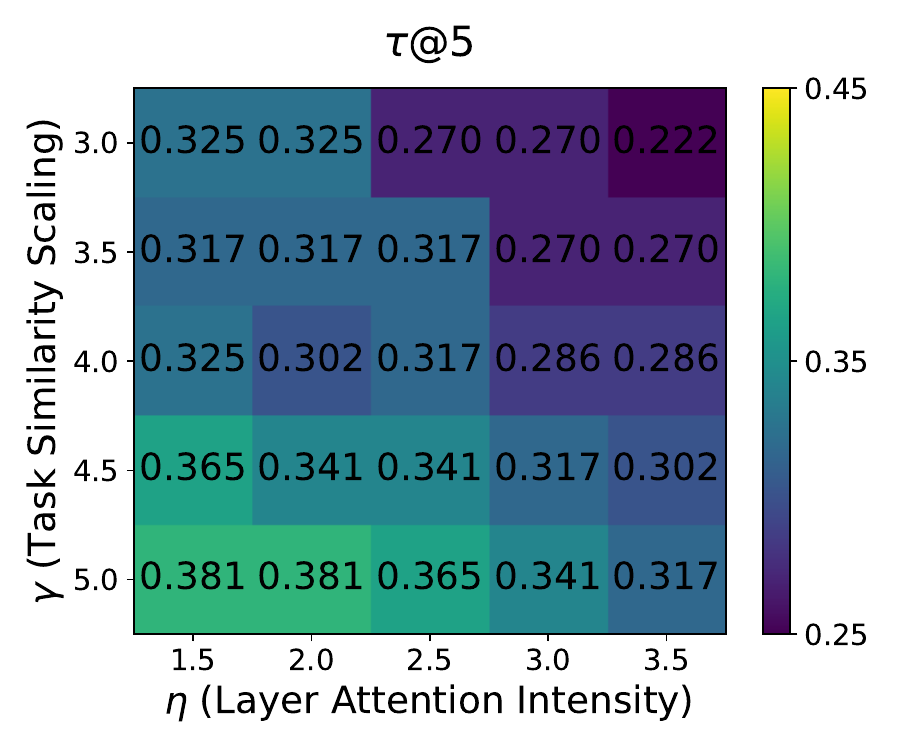}
    \\[2pt]
    \small (a) $\tau$@5 (higher is better)
  \end{minipage}
  \hfill
  \begin{minipage}[t]{0.32\textwidth}
    \centering
    \includegraphics[width=\linewidth]{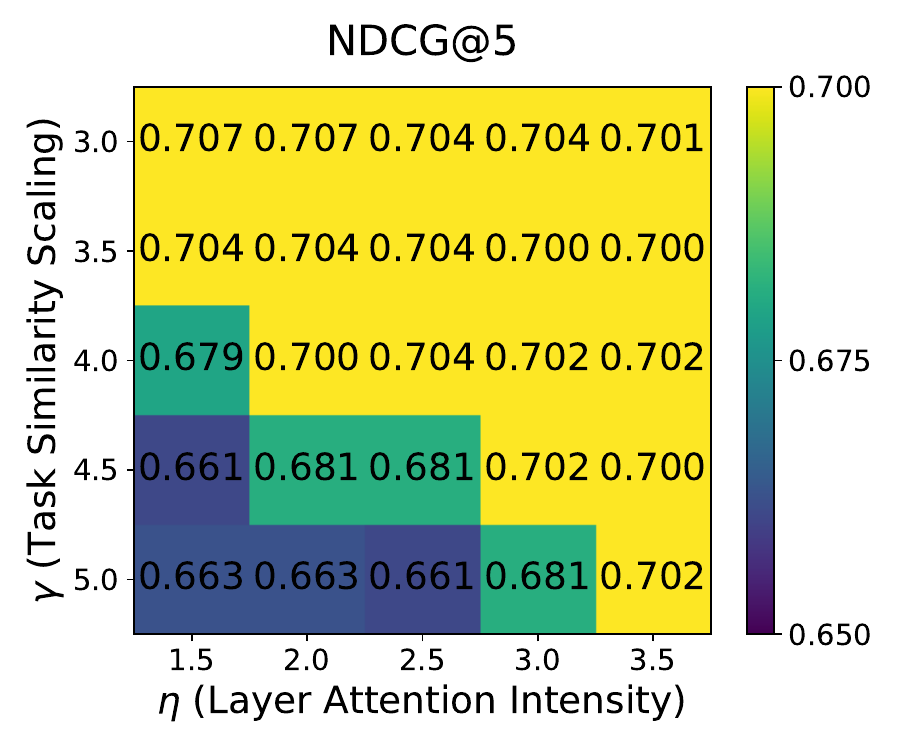}
    \\[2pt]
    \small (b) NDCG@5 (higher is better)
  \end{minipage}
  \hfill
  \begin{minipage}[t]{0.32\textwidth}
    \centering
    \includegraphics[width=\linewidth]{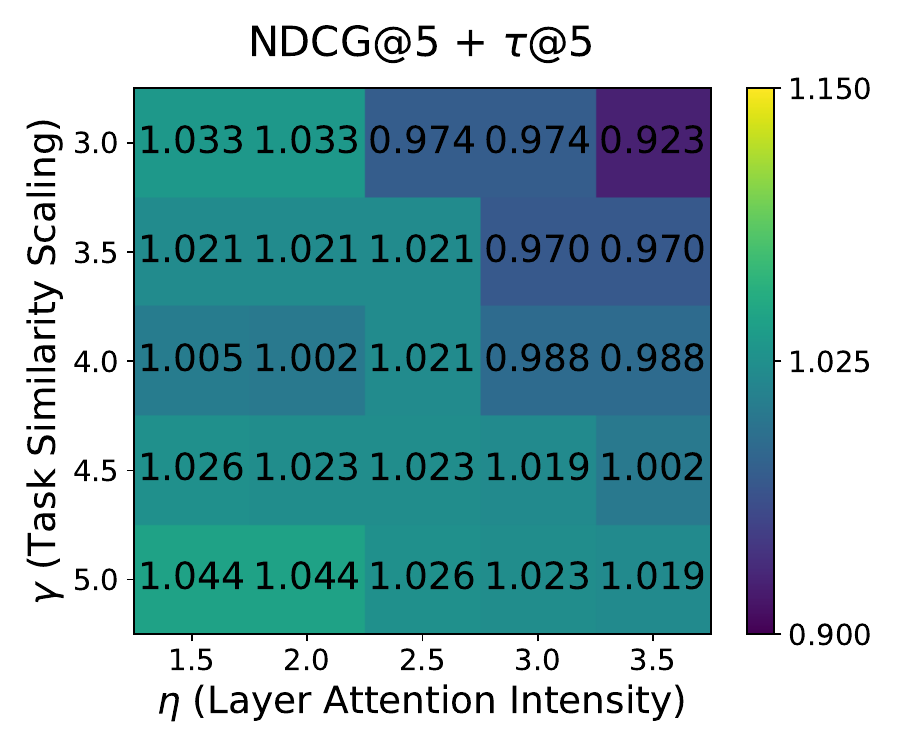}
    \\[2pt]
    \small (c) Sum = NDCG@5 + $\tau$@5
  \end{minipage}

  \caption{\textbf{Sensitivity to $\eta$ and $\gamma$.} 
Heatmaps of selection performance when varying $\eta$ (softmax temperature for target-conditioned block importance) and $\gamma$ (softmin temperature for DCD-based source weighting). 
All values are averaged over 10 independent random low-shot sampling runs.
Performance is stable across a wide range.}

  \label{fig:hyperparameter}
\end{figure*}
Our method incorporates two key hyperparameters: $\eta$ (layer attention intensity), which regulates the softmax temperature to weight network blocks based on their conductance, and $\gamma$, which scales task similarity to modulate discriminative strength when aggregating rankings from source tasks. We evaluate the sensitivity of our approach by varying $\eta \in \{1.5, 2.0, 2.5, 3.0, 3.5\}$ and $\gamma \in \{3.0, 3.5, 4.0, 4.5, 5.0\}$. The results are visualized as heatmaps in Figure~\ref{fig:hyperparameter}.

The empirical results demonstrate that our method is highly robust to hyperparameter variations. Specifically, varying $\eta$ from 1.5 to 3.5 yields only marginal fluctuations in NDCG@5, suggesting that task representations remain discriminative provided that informative layers are sufficiently prioritized. Similarly, performance remains stable across the tested range of $\gamma$, with $\tau$@5 ranging from 0.222 to 0.381. The configuration of $\eta=2.0$ and $\gamma=5.0$ attains the peak $\tau$@5 (0.365), striking an effective balance between layer-wise noise suppression and task discrimination. These findings confirm that our method maintains strong performance across a broad parameter space without the need for exhaustive per-task tuning.

\subsection{Efficiency of Task Sampling}
\label{sec:efficiency}

We separately vary the number of unlabeled images for source-task and target-task characterization. 
For source-task sampling, we fix $N_{\text{tgt}}=1$ and vary $N_{\text{src}}$ from 1 to 100, averaging over 5 independent random low-shot sampling runs. 
As shown in Figure~\ref{fig:sample_efficiency}(a), performance improves rapidly when $N_{\text{src}}$ increases to 25 and then saturates around 0.71 NDCG@5 and 0.38 $\tau$@5. 
This suggests that a compact source library with about 25 anchor images per task is sufficient and reusable across target tasks.

We further vary $N_{\text{tgt}}$ while keeping the source library fixed. 
As shown in Figure~\ref{fig:sample_efficiency}(b), performance stays within a small fluctuation range across target sample sizes, with NDCG@5 close to 0.70 and competitive $\tau$@5 even in the single-shot setting. 
Thus, very few target images already provide informative model-specific conductance signals, and increasing $N_{\text{tgt}}$ brings limited additional gain. 
Small fluctuations at larger $N_{\text{tgt}}$ should not be over-interpreted, as part of the variation may come from sampling noise.

\begin{figure*}[t]
  \centering
  \begin{minipage}[t]{0.48\textwidth}
    \centering
    \includegraphics[width=\linewidth]{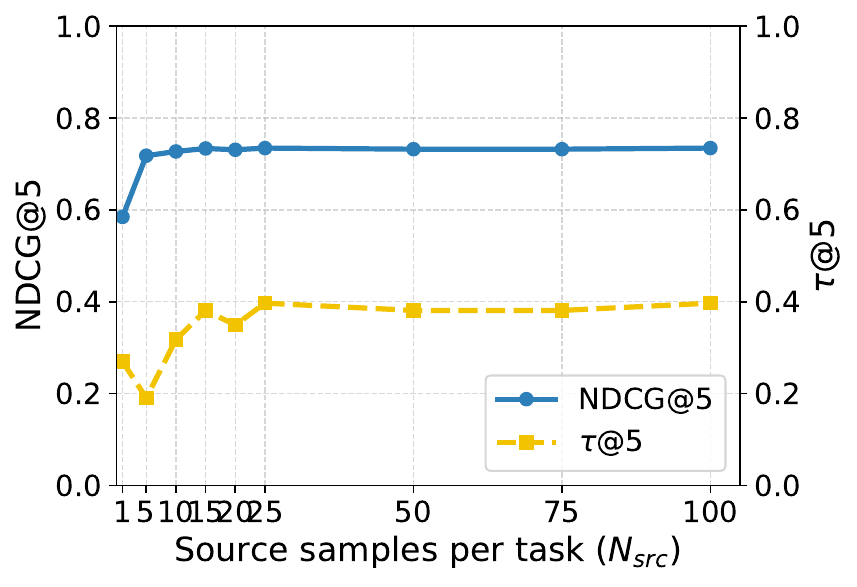}
    \\[2pt]
    \small (a) Source samples per task ($N_{\text{src}}$).
  \end{minipage}
  \hfill
  \begin{minipage}[t]{0.48\textwidth}
    \centering
    \includegraphics[width=\linewidth]{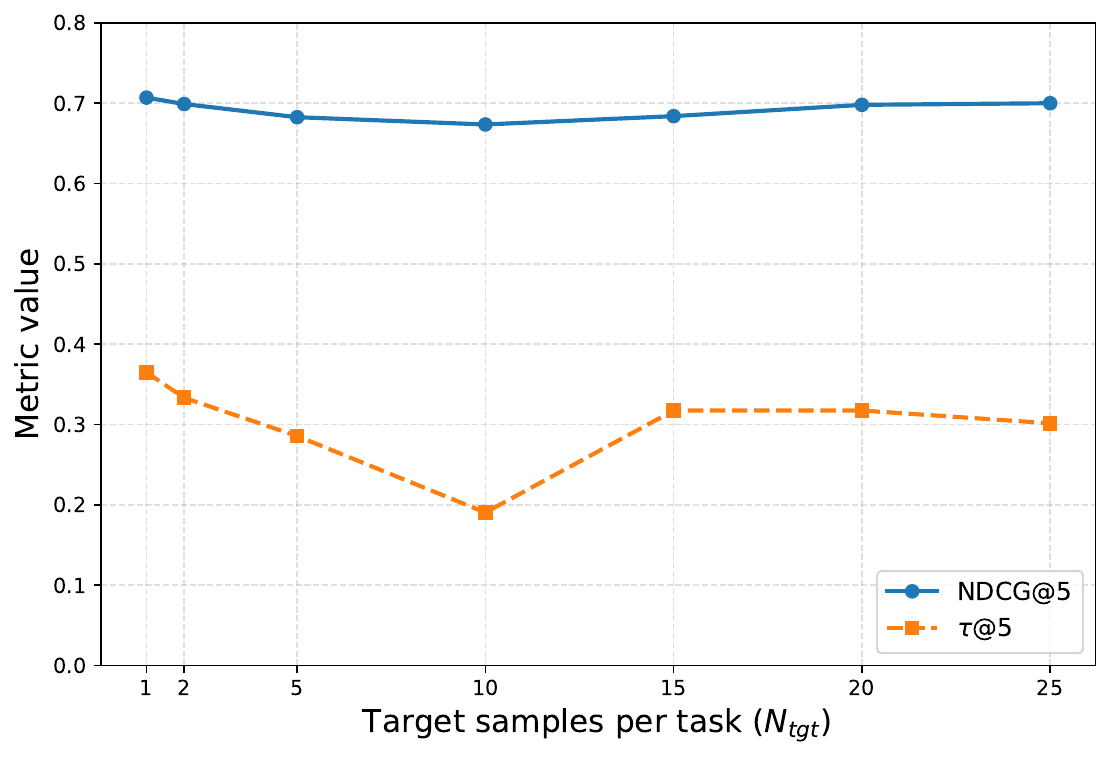}
    \\[2pt]
    \small (b) Target samples per task ($N_{\text{tgt}}$).
  \end{minipage}
  \caption{\textbf{Efficiency of Task Sampling.} 
Source-task sampling saturates around $N_{\text{src}}=25$, while target-task sampling remains stable with small fluctuations across different $N_{\text{tgt}}$ values.}
\label{fig:sample_efficiency}
\end{figure*}

\subsection{Model Clustering in Conductance induced Task Space}
\label{ap:model_clustering}

To examine whether conductance-induced task similarity merely reflects architecture identity, we cluster all 48 models using HDBSCAN on their task-similarity structures. 
For each model, we compute conductance-based task representations with 25 unlabeled samples per task, derive its pairwise similarity matrix over 21 tasks, flatten the matrix, and apply softmax normalization before clustering.

Table~\ref{tab:task_cluster} shows that the induced clusters partly follow architecture and pretraining patterns, but are not determined by architecture alone. 
For example, several clusters mix different families, including ConvNeXt with EVA, NLLB-CLIP with ViT, and text-variant models with ViT. 
This suggests that the conductance-induced task space captures model-dependent task perception beyond coarse architecture identity, while still being influenced by architecture and pretraining.

\begin{table*}[p]
\centering
\caption{\textbf{Model clustering in conductance-induced task space.}
HDBSCAN clustering is performed on softmax-normalized task-similarity vectors derived from layer conductance. 
Cluster $-1$ denotes noise points. Clusters marked with $\star$ contain multiple architecture families.}
\label{tab:task_cluster}

\begin{adjustbox}{width=\textwidth,totalheight=0.92\textheight,keepaspectratio,center}
\begin{tabular}{clll}
\toprule
\textbf{Cluster} & \textbf{Family} & \textbf{Architecture} & \textbf{Model} \\
\midrule
\multirow{15}{*}{$-1^\dagger$} 
 & CoCa & CoCa & \texttt{coca\_vit\_b\_32\_laion2b\_s13b\_b90k} \\
 & CoCa & CoCa & \texttt{coca\_vit\_l\_14\_laion2b\_s13b\_b90k} \\
 & EVA & EVA & \texttt{eva02\_b\_16\_merged2b\_s8b\_b131k} \\
 & EVA & EVA & \texttt{eva02\_l\_14\_336\_merged2b\_s6b\_b61k} \\
 & ResNet & ResNet-50x & \texttt{rn50x64\_openai} \\
 & ViT & ViT-B/16 & \texttt{vit\_b\_16\_dfn2b} \\
 & ViT & ViT-B/16 & \texttt{vit\_b\_16\_laion2b\_s34b\_b88k} \\
 & ViT & ViT-B/16 & \texttt{vit\_b\_16\_plus\_240\_laion400m\_e31} \\
 & ViT & ViT-B/16 & \texttt{vit\_b\_16\_quickgelu\_metaclip\_fullcc} \\
 & ViT & ViT-B/32 & \texttt{vit\_b\_32\_256\_datacomp\_s34b\_b86k} \\
 & ViT & ViT-B/32 & \texttt{vit\_b\_32\_commonpool\_m\_clip\_s128m...} \\
 & ViT & ViT-B/32 & \texttt{vit\_b\_32\_quickgelu\_metaclip\_fullcc} \\
 & ViT & ViT-bigG & \texttt{vit\_bigg\_14\_laion2b\_s39b\_b160k} \\
 & ViT & ViT-H & \texttt{vit\_h\_14\_quickgelu\_metaclip\_fullcc} \\
 & ViT & ViT-L & \texttt{vit\_l\_14\_quickgelu\_metaclip\_fullcc} \\
\midrule
\multirow{3}{*}{0} 
 & ResNet & ResNet-101 & \texttt{rn101\_quickgelu\_yfcc15m} \\
 & ResNet & ResNet-101 & \texttt{rn101\_yfcc15m} \\
 & ResNet & ResNet-50 & \texttt{rn50\_cc12m} \\
\midrule
\multirow{4}{*}{1} 
 & ResNet & ResNet-101 & \texttt{rn101\_openai} \\
 & ResNet & ResNet-101 & \texttt{rn101\_quickgelu\_openai} \\
 & ResNet & ResNet-50 & \texttt{rn50\_openai} \\
 & ResNet & ResNet-50x & \texttt{rn50x4\_openai} \\
\midrule
\multirow{2}{*}{2} 
 & ConvNeXt & ConvNeXt & \texttt{convnext\_base\_w\_320\_laion\_aesthe...} \\
 & ConvNeXt & ConvNeXt & \texttt{convnext\_base\_w\_laion2b\_s13b\_b82k} \\
\midrule
\multirow{4}{*}{3$\star$} 
 & ConvNeXt & ConvNeXt & \texttt{convnext\_large\_d\_320\_laion2b\_s29...} \\
 & ConvNeXt & ConvNeXt & \texttt{convnext\_large\_d\_laion2b\_s26b\_b1...} \\
 & ConvNeXt & ConvNeXt & \texttt{convnext\_xxlarge\_laion2b\_s34b\_b8...} \\
 & EVA & EVA & \texttt{eva01\_g\_14\_laion400m\_s11b\_b41k} \\
\midrule
\multirow{6}{*}{4$\star$} 
 & NLLB-CLIP & NLLB-CLIP & \texttt{nllb\_clip\_base\_siglip\_v1} \\
 & ViT & ViT-B/16 & \texttt{vit\_b\_16\_openai} \\
 & ViT & ViT-H & \texttt{vit\_h\_14\_378\_quickgelu\_dfn5b} \\
 & ViT & ViT-L & \texttt{vit\_l\_14\_336\_openai} \\
 & ViT & ViT-L & \texttt{vit\_l\_14\_openai} \\
 & ViT & ViT-L & \texttt{vit\_l\_14\_quickgelu\_dfn2b} \\
\midrule
\multirow{2}{*}{5} 
 & ViT & ViT-B/16 & \texttt{vit\_b\_16\_commonpool\_l\_laion\_s1b\_b8k} \\
 & ViT & ViT-B/16 & \texttt{vit\_b\_16\_datacomp\_l\_s1b\_b8k} \\
\midrule
\multirow{8}{*}{6$\star$} 
 & Text-variant & XLM-RoBERTa & \texttt{xlm\_roberta\_base\_vit\_b\_32\_laion5...} \\
 & ViT & ViT-B/32 & \texttt{vit\_b\_32\_datacomp\_xl\_s13b\_b90k} \\
 & ViT & ViT-B/32 & \texttt{vit\_b\_32\_openai} \\
 & ViT & ViT-G & \texttt{vit\_g\_14\_laion2b\_s12b\_b42k} \\
 & ViT & ViT-H & \texttt{vit\_h\_14\_laion2b\_s32b\_b79k} \\
 & ViT & ViT-L & \texttt{vit\_l\_14\_commonpool\_xl\_clip\_s13b...} \\
 & ViT & ViT-L & \texttt{vit\_l\_14\_datacomp\_xl\_s13b\_b90k} \\
 & ViT & ViT-L & \texttt{vit\_l\_14\_laion400m\_e31} \\
\midrule
\multirow{2}{*}{7$\star$} 
 & Text-variant & RoBERTa & \texttt{roberta\_vit\_b\_32\_laion2b\_s12b\_b32k} \\
 & ViT & ViT-B/32 & \texttt{vit\_b\_32\_laion2b\_e16} \\
\midrule
\multirow{2}{*}{8$\star$} 
 & NLLB-CLIP & NLLB-CLIP & \texttt{nllb\_clip\_base\_v1} \\
 & ViT & ViT-B/32 & \texttt{vit\_b\_32\_quickgelu\_laion400m\_e31} \\
\bottomrule
\end{tabular}
\end{adjustbox}
\end{table*}

\subsection{Variation of the Scalar Probe}
\label{ap:scalar_probe_variation}

To check whether the scalar probe $F_m(x)=\|f_m(x)\|_2$ is overly normalized, we report its statistics across representative models and tasks in Table~\ref{tab:scalar_probe_variation}. 
For each model--task pair, we sample 50 images and compute the mean, standard deviation, coefficient of variation (CV), minimum, and maximum. 
The values vary across models and tasks, indicating that the probe is not a degenerate constant. 
We use $F_m(x)$ only to induce conductance profiles, not as a direct proxy for downstream accuracy or transferability.

\begin{table*}[p]
\centering
\scriptsize
\setlength{\tabcolsep}{4pt}
\renewcommand{\arraystretch}{0.95}
\caption{\textbf{Variation of the scalar probe across models and tasks.}
We report statistics of $F_m(x)=\|f_m(x)\|_2$ across 5 representative models and 6 tasks, with 50 images per model--task pair. 
CV denotes the coefficient of variation, i.e., $\sigma/\mu$.}
\label{tab:scalar_probe_variation}
\begin{tabular*}{\textwidth}{@{\extracolsep{\fill}} llccccc}
\toprule
\textbf{Model} & \textbf{Task} & \textbf{Mean} & \textbf{Std} & \textbf{CV} & \textbf{Min} & \textbf{Max} \\
\midrule
\multirow{6}{*}{ViT-B-16}
 & CIFAR-100 & 11.28 & 0.2904 & 0.0257 & 10.34 & 11.81 \\
 & MNIST     & 11.51 & 0.2220 & 0.0193 & 11.07 & 11.96 \\
 & STL-10    & 11.11 & 0.4540 & 0.0409 & 10.18 & 12.10 \\
 & SVHN      & 11.47 & 0.4701 & 0.0410 & 10.52 & 12.57 \\
 & GTSRB     & 11.40 & 0.4302 & 0.0377 & 10.48 & 12.31 \\
 & DTD       & 11.41 & 0.3385 & 0.0297 & 10.75 & 12.09 \\
\midrule
\multirow{6}{*}{ViT-L-14}
 & CIFAR-100 & 20.11 & 0.9469 & 0.0471 & 17.62 & 21.94 \\
 & MNIST     & 21.56 & 0.4689 & 0.0218 & 20.50 & 22.95 \\
 & STL-10    & 19.97 & 0.8796 & 0.0441 & 18.45 & 21.94 \\
 & SVHN      & 20.87 & 0.5883 & 0.0282 & 19.51 & 22.79 \\
 & GTSRB     & 19.43 & 0.8268 & 0.0425 & 17.56 & 20.92 \\
 & DTD       & 19.91 & 0.6529 & 0.0328 & 18.38 & 21.26 \\
\midrule
\multirow{6}{*}{RN50}
 & CIFAR-100 & 2.15 & 0.1762 & 0.0819 & 1.83 & 2.55 \\
 & MNIST     & 2.10 & 0.1030 & 0.0490 & 1.89 & 2.33 \\
 & STL-10    & 2.07 & 0.2104 & 0.1017 & 1.63 & 2.63 \\
 & SVHN      & 2.17 & 0.1123 & 0.0517 & 1.92 & 2.43 \\
 & GTSRB     & 1.95 & 0.1443 & 0.0738 & 1.65 & 2.23 \\
 & DTD       & 1.89 & 0.2267 & 0.1197 & 1.62 & 2.79 \\
\midrule
\multirow{6}{*}{convnext\_base\_w}
 & CIFAR-100 & 7.93 & 0.5370 & 0.0677 & 6.88 & 9.51 \\
 & MNIST     & 7.51 & 0.1971 & 0.0263 & 7.07 & 7.98 \\
 & STL-10    & 8.53 & 0.4890 & 0.0573 & 7.53 & 9.49 \\
 & SVHN      & 7.25 & 0.3494 & 0.0482 & 6.69 & 8.26 \\
 & GTSRB     & 7.59 & 0.5462 & 0.0720 & 6.52 & 8.73 \\
 & DTD       & 9.31 & 0.8790 & 0.0944 & 7.81 & 11.50 \\
\midrule
\multirow{6}{*}{EVA02-B-16}
 & CIFAR-100 & 0.12 & 0.0079 & 0.0682 & 0.11 & 0.15 \\
 & MNIST     & 0.16 & 0.0129 & 0.0796 & 0.14 & 0.19 \\
 & STL-10    & 0.11 & 0.0049 & 0.0452 & 0.10 & 0.12 \\
 & SVHN      & 0.14 & 0.0123 & 0.0877 & 0.11 & 0.17 \\
 & GTSRB     & 0.11 & 0.0076 & 0.0671 & 0.10 & 0.13 \\
 & DTD       & 0.13 & 0.0096 & 0.0730 & 0.11 & 0.16 \\
\bottomrule
\end{tabular*}
\end{table*}

\section{Miscellaneous Details}
\subsection{Computational Cost}
\label{ap:cost}

Conductance extraction is more expensive than a standard forward-only pass, and our method is not designed for raw latency superiority. 
Its practical advantage lies in source-side reusability and low-supervision target-side selection. 
The cost decomposes into an offline stage, where source-task conductance representations are computed once and reused across target tasks, and an online stage, where only a few unlabeled target samples are used to compute target conductance before lightweight rank aggregation.

In our implementation, extracting target conductance across 48 VLMs with one target sample takes approximately 11.2 hours. 
Increasing the target sample size to 150 takes approximately 180 hours, which is sub-linear in the sample count due to fixed overhead such as model loading. 
By comparison, standard forward evaluation on 5 labeled target images across all candidate models takes approximately 2.7 hours, but requires target labels and does not benefit from source-side amortization.

\subsection{Distribution over 10 Runs for Main Results}
\label{ap:complete_results}
Figure~\ref{fig:main_boxplots} shows the boxplots of the main results over 10 random runs.
\begin{figure*}[t]
  \centering
  \begin{minipage}[t]{0.48\textwidth}
    \centering
    \includegraphics[width=\linewidth]{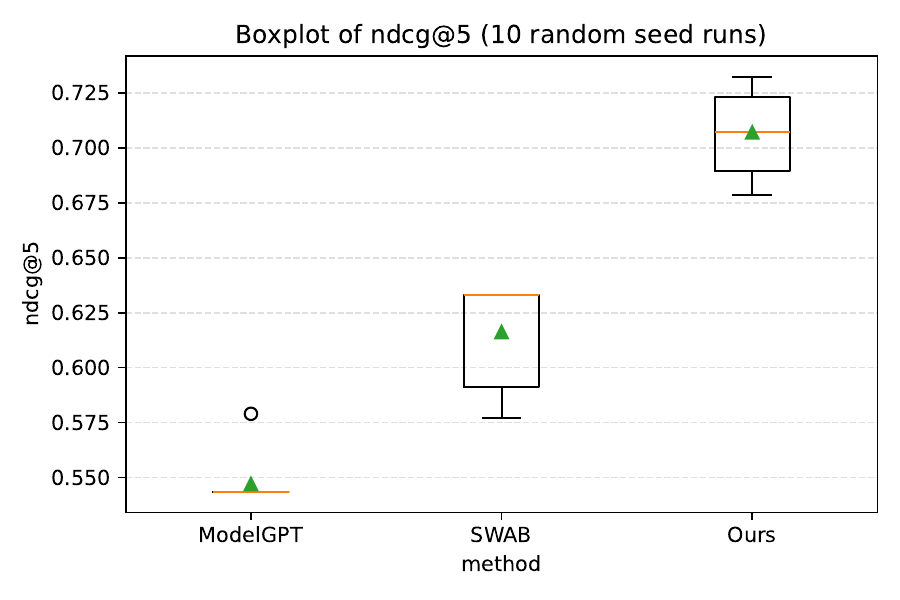}
    \\[2pt]
    \small (a) NDCG@5 over 10 runs.
  \end{minipage}
  \hfill
  \begin{minipage}[t]{0.48\textwidth}
    \centering
    \includegraphics[width=\linewidth]{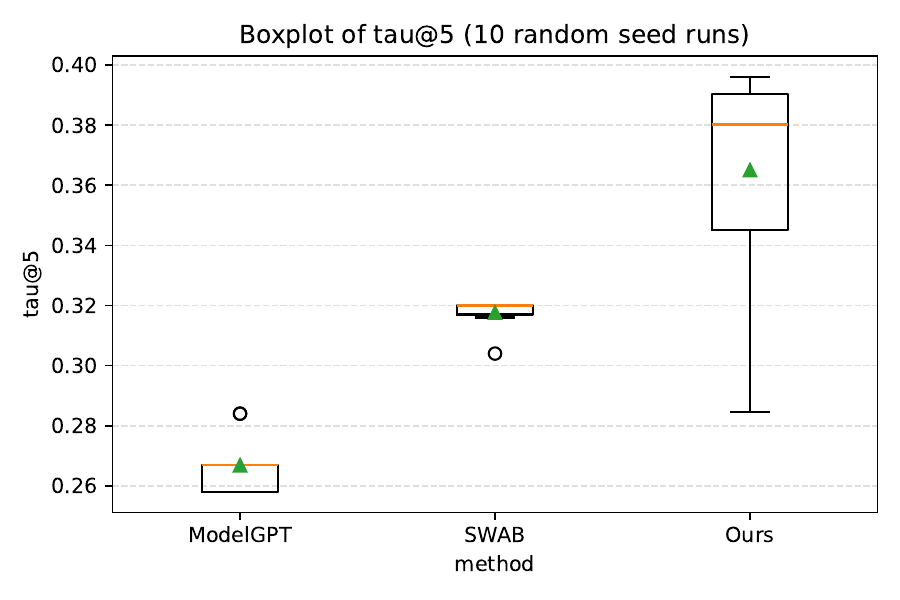}
    \\[2pt]
    \small (b) Kendall's $\tau$@5 over 10 runs.
  \end{minipage}
  \caption{Boxplots of main results over 10 random sampling runs. Each box summarizes the distribution across runs for three methods (ModelGPT, SWAB, and ours).}
  \label{fig:main_boxplots}
\end{figure*}

\subsection{Evaluation Metrics}
\label{ap:metrics}
For each task $\tau$, we compare estimated ranking $\widehat{\pi}_\tau$ (from $\widehat{R}_m(\tau)$, Sec.~\ref{sec:dcd}) against ground-truth $\pi_\tau^\ast$ (from $R_m(\tau)$) on the Top-k intersection
\begin{equation}
\begin{aligned}
\mathcal{I}_\tau^{k} \triangleq {}& \{m \mid \hat{r}_\tau(m) \le k\} \\
&\cap \{m \mid r_\tau^\ast(m) \le k\},
\end{aligned}
\end{equation}
where $\hat{r}_\tau(m)$ and $r_\tau^\ast(m)$ denote predicted and ground-truth ranks. We report NDCG@k, Kendall's $\tau$@k (SciPy implementation with tie-handling), and their sum; both metrics are set to $0$ when $|\mathcal{I}_\tau^{k}|<2$.

\paragraph{NDCG@k.}
Define relevance $\mathrm{rel}_\tau(m)\triangleq \mathrm{rank}_{\downarrow}(r_\tau^\ast(m))$. Ordering $\mathcal{I}_\tau^{k}$ by $\hat{r}_\tau(m)$ as $(m_{(1)},\dots,m_{(n)})$,
\begin{equation}
\mathrm{NDCG@k}(\tau)\triangleq
\frac{\sum_{i=1}^{n}\frac{2^{\mathrm{rel}_\tau(m_{(i)})}-1}{\log_2(i+1)}}{\mathrm{IDCG@k}(\tau)}.
\end{equation}
where $n=|\mathcal{I}_\tau^{k}|$ and $\mathrm{IDCG@k}(\tau)$ is the ideal DCG@k.

\paragraph{$\tau$@k.}
\begin{equation}
\tau@k(\tau)\triangleq
\frac{C-D}{\binom{n}{2}},
\end{equation}
where $C$ and $D$ count concordant and discordant pairs between predicted and ground-truth ranks on $\mathcal{I}_\tau^{k}$.

\paragraph{Sum.}
$\mathrm{Score}(\tau)\triangleq \mathrm{NDCG@5}(\tau)+\tau@5(\tau)$.

\subsection{Ground-Truth Rankings and Conductance Examples}
To characterize the dataset-dependent nature of model performance, we summarize the ground-truth ranking of all models on each dataset.Figure~\ref{fig:ground} shows that the induced model order is not universal. We further visualize layer conductance patterns aggregated over all datasets for two representative models (Figure~\ref{fig:layer_con_example}), including coca-ViT-L-14 (laion2b\_s13b\_b90k) and xlm-roberta-base-ViT-B-32 (laion5b\_s13b\_b90k).
\begin{figure*}[t]
  \centering
  \begin{minipage}[t]{0.48\textwidth}
    \centering
    \includegraphics[width=\linewidth]{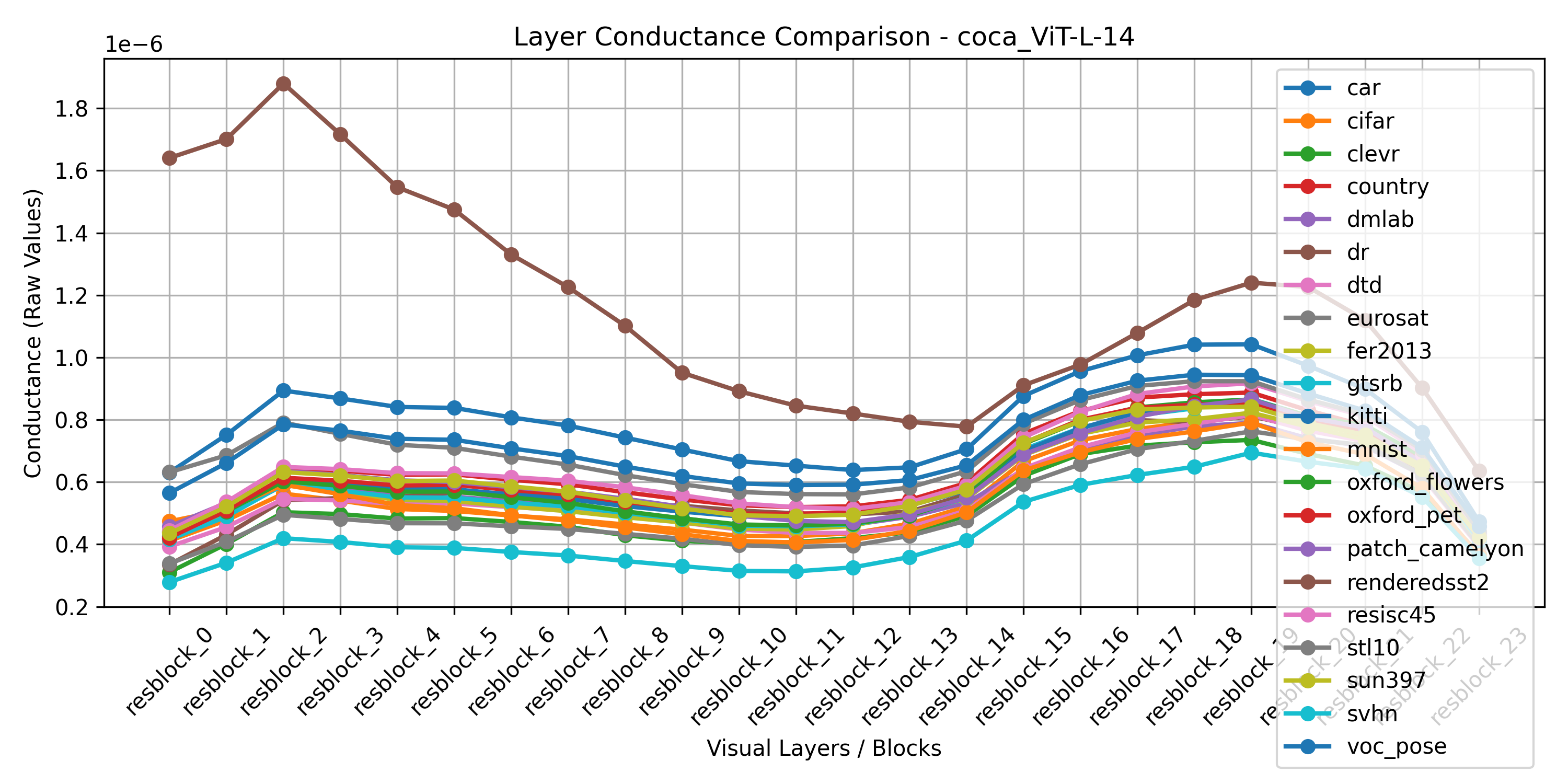}
    \\[2pt]
    \small (a) coca-ViT-L-14 (laion2b\_s13b\_b90k).
  \end{minipage}
  \hfill
  \begin{minipage}[t]{0.48\textwidth}
    \centering
    \includegraphics[width=\linewidth]{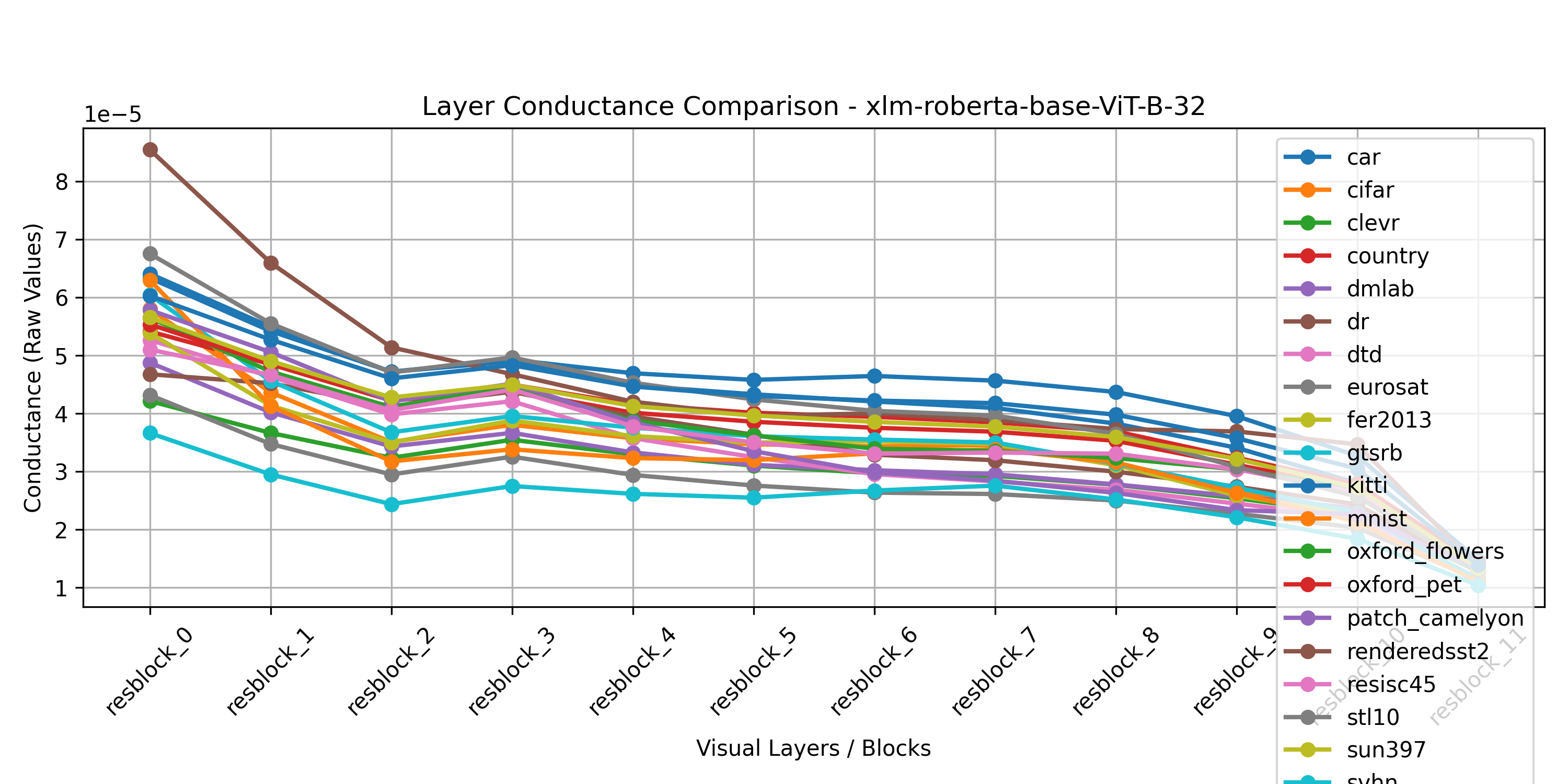}
    \\[2pt]
    \small (b) xlm-roberta-base-ViT-B-32 (laion5b\_s13b\_b90k).
  \end{minipage}
  \caption{Layer conductance patterns aggregated over all datasets for two representative models.}
  \label{fig:layer_con_example}
\end{figure*}

\begin{figure*}[t]
  \centering
  \includegraphics[width=0.82\textwidth]{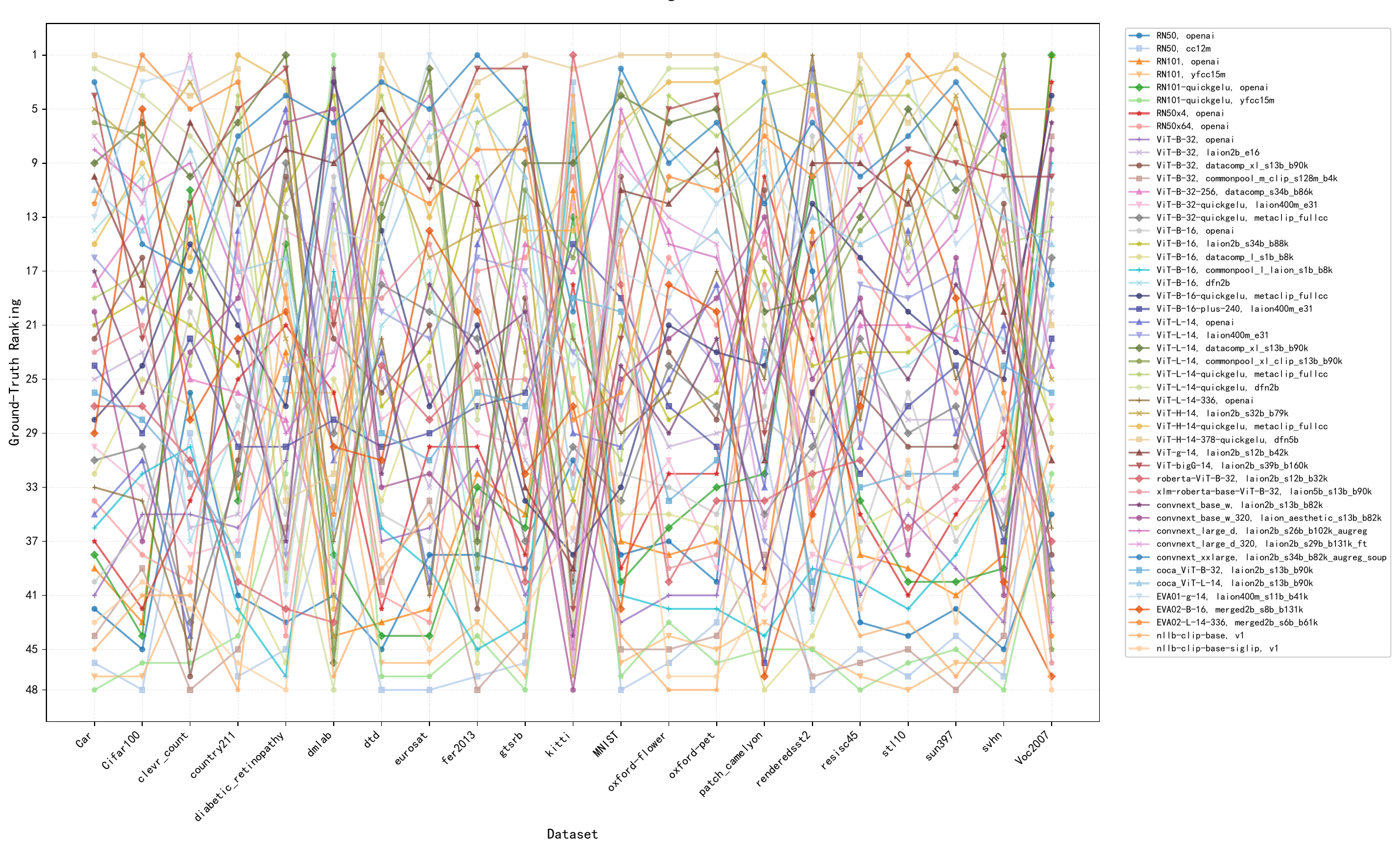}
  \caption{Ground-truth model ranking across all datasets. The ranking order varies across datasets.}
  \label{fig:ground}
\end{figure*}

\end{document}